\pdfoutput=1

\documentclass[english,11pt]{article}
\include{macros}

\begin{document}

\title{\Large A Minimalist Bayesian Framework for Stochastic Optimization}

\author{Kaizheng Wang\thanks{Department of IEOR and Data Science Institute, Columbia University. Email: \texttt{kaizheng.wang@columbia.edu}.}
}

\date{This version: \today}

\maketitle

\begin{abstract}
The Bayesian paradigm offers principled tools for sequential decision-making under uncertainty, but its reliance on a probabilistic model for all parameters can hinder the incorporation of complex structural constraints. We introduce a minimalist Bayesian framework that places a prior only on the component of interest, such as the location of the optimum. Nuisance parameters are eliminated via profile likelihood, which naturally handles constraints. As a direct instantiation, we develop a MINimalist Thompson Sampling (MINTS) algorithm. Our framework accommodates structured problems, including continuum-armed Lipschitz bandits and dynamic pricing. It also provides a probabilistic lens on classical convex optimization algorithms such as the center of gravity and ellipsoid methods. We further analyze MINTS for multi-armed bandits and establish near-optimal regret guarantees.
\end{abstract}
\noindent{\bf Keywords:} Stochastic optimization, Bayesian method, profile likelihood, regret analysis.

\section{Introduction}\label{sec-intro}

Sequential decision-making under uncertainty is a ubiquitous challenge, where an agent repeatedly make decisions to optimize an unknown objective function based on limited, noisy feedback. Effective performance requires striking a delicate balance between exploiting actions that are believed to be optimal based on past data and exploring lesser-known actions to gather new information.

Several paradigms have emerged to address this exploration-exploitation tradeoff. The \emph{explore-then-commit} (ETC) strategy is the most straightforward, dividing the time horizon into distinct exploration and exploitation phases. Yet, it can be difficult to decide an appropriate split. A more adaptive paradigm is the principle of \emph{optimism in the face of uncertainty}, which constructs optimistic indices like upper confidence bounds to encourage exploration \citep{LRo85,ACF02}. While theoretically sound, this approach often relies on problem-specific bonus calibration. A third one is the \emph{Bayesian paradigm}, which treats the problem as a random instance and maintains a posterior distribution to quantify uncertainty \citep{Tho33,JSW98}. The belief update powered by Bayes' rule provides a principled mechanism for automatically balancing exploration and exploitation. However, the standard Bayesian paradigm requires specifying a prior for all unknown parameters. This becomes a significant bottleneck when one wishes to encode rich structural knowledge, such as shape or smoothness constraints on the objective function, as it can be prohibitively difficult to design a tractable prior that is faithful to these constraints.

To resolve this dilemma, we introduce a minimalist Bayesian framework with significantly enhanced flexibility. Our approach allows the user to place a prior only on a low-dimensional component of interest, such as the location of the optimum. All other parameters are treated as nuisance and handled by the profile likelihood method \citep{BCo94}. We can conveniently enforce structural constraints on them. The reduced-dimension prior and profile likelihood yield a generalized posterior for the component of interest, which then guides subsequent decisions.

\paragraph{Main contributions}

Our contributions are threefold. 
\begin{itemize}
\item (Framework) We develop a minimalist Bayesian framework for stochastic optimization that reasons about the parameter of interest without probabilistic modeling of all unknowns. The approach is lightweight and naturally accommodates structural constraints.

\item (Algorithms and insights) We derive a MINimalist Thompson Sampling (\Alg) algorithm that directly updates and samples from the posterior of the optimum. We also instantiate the framework in complex structured problems including Lipschitz bandits and dynamic pricing, and provide novel probabilistic interpretations of classical convex optimization algorithms.

\item (Theory) We analyze \Alg~for multi-armed bandits and derive near-optimal regret bounds. The theoretical results rigorously justify the effectiveness of our new framework.
\end{itemize}

\paragraph{Related work}

Our framework is closely related to Thompson sampling (TS) \citep{Tho33}, also known as posterior sampling or probability matching because the decision is drawn from the posterior distribution of the optimum \citep{Sco10,RVK18}. However, TS derives such posterior indirectly through a probabilistic model for the entire problem instance. For multi-armed bandits, this means placing a prior on the expected rewards of all arms. In contrast, \Alg~directly specifies a prior on which arm is optimal. A related idea appears in online learning with full information, where one can place a prior on the optimum, exponentiate the empirical loss to form a ``likelihood,'' and sample from a Gibbs posterior \citep{LWa89,Vov90,CFH97,Cat04,BHW16}. These methods require feedback for all decisions in each round, whereas we address the more challenging partial feedback setting.

Certain stochastic optimization problems permit Bayesian reasoning about the optimum without modeling the entire objective, such as probabilistic bisection search over an interval \citep{WFH13}. For general problems, several Bayesian algorithms choose actions to maximize information gain about the optimal solution, by increasing the mutual information between the next observation and the optimum  \citep{VVW09,HSc12,HHG14,RVa18}. These approaches still rely on full probabilistic models for belief updates. Closer in spirit to our work, \citet{SNO21} combine a user-specified prior on the optimum with a standard Bayesian optimization model (e.g.,~Gaussian process) to form a pseudo-posterior, whereas we use profile likelihood to circumvent modeling of the nuisance parameters.

Finally, our framework provides a unified approach to Bayesian optimization under structural constraints, a domain that has largely been addressed case by case. In continuous black-box optimization, efforts have focused on designing Gaussian processes tailored to particular shape constraints \citep{SGF20}. In structured bandits, \cite{VGo24} developed a computationally tractable frequentist algorithm with strong guarantees. It handles general convex constraints but requires the reward distributions of all arms to be supported on a common finite set. Bayesian counterparts with comparable flexibility and practicality are missing. Current approaches focus on specific settings, such as unimodal bandits on graphs \citep{PTR17}.

\paragraph{Outline} The rest of the paper is organized as follows. \Cref{sec-preliminaries} lays out the preliminaries of stochastic optimization. \Cref{sec-method} introduces the minimalist Bayesian framework and the \Alg~algorithm. \Cref{sec-examples} instantiates the framework on canonical problems. \Cref{sec-experiments} demonstrates its efficacy through numerical experiments. \Cref{sec-MAB} conducts a theoretical analysis of \Alg~on multi-armed bandits. \Cref{sec-discussion} concludes the paper with a discussion of future directions.

\paragraph{Notation}
We use the symbol $[n]$ as a shorthand for $\{ 1, 2, \cdots, n \}$ and $| \cdot |$ to denote the absolute value of a real number or cardinality of a set. For nonnegative sequences $\{ a_n \}_{n=1}^{\infty}$ and $\{ b_n \}_{n=1}^{\infty}$, we write $a_n \lesssim b_n$ if there exists a positive constant $C$ such that $a_n \leq C b_n$. 


\section{Preliminaries}\label{sec-preliminaries}

This section lays out the preliminaries for our analysis. We begin by formally defining the stochastic optimization framework and presenting several canonical examples that illustrate its scope. Then, we review the standard Bayesian paradigm for these problems and highlight the key difficulties in incorporating structural knowledge, which motivates the new framework developed in this paper.

\subsection{Stochastic optimization}\label{sec-preliminaries-setup}

In the standard setup of stochastic optimization, an agent seeks to maximize an unknown objective function $f$ over a decision set $\spaceofaction$:
\begin{align}
	\max_{\action \in \spaceofaction} f (\action).
\label{eqn-problem}
\end{align}
The agent learns about the optimum by sequentially interacting with the environment. Starting with an empty dataset $\dataset_0 = \varnothing$, at each period $t \in \ZZ_+$, the agent selects a decision $\action_t \in \spaceofaction$ based on past data $\dataset_{t-1}$, receives randomized feedback $\feedback_t$ from the environment, and updates the dataset to $\dataset_{t} = \dataset_{t-1} \cup \{ (\action_t, \feedback_t) \}$. The performance over $T$ time periods is typically measured by either the cumulative regret $ 
\sum_{t=1}^T 
[
\max_{\action \in \spaceofaction} f(\action)
-
f ( \action_t ) ]$, which sums the suboptimality of each action, or the simple regret $\max_{\action \in \spaceofaction} f(\action) - f (\action_T)$, which measures the quality of the final action taken. Below we present several common examples. 

\begin{example}[Multi-armed bandit]\label{example-MAB}
The decision set is a collection of $K \in \ZZ_+$ arms, i.e.~$\spaceofaction = [K]$. Each arm $\action$ is associated with a reward distribution $ \distP_{\action} $ over $\RR$, and the objective value $f(\action)$ is the expected reward $  \EE_{\rewardrv \sim \distP_{\action}} \rewardrv$. Given $\action_t $ and $\dataset_{t-1}$, the feedback $\feedback_t$ is a sample from $\distP_{\action_t}$.
\end{example}

\begin{example}[Lipschitz bandit \citep{KSU08}]\label{example-Lipschitz}
The set $\spaceofaction$ is equipped with a metric $\distance$. Each decision $\action$ is associated with a reward distribution $ \distP_{\action} $ over $\RR$, and the objective value $f(\action)$ is the expected reward $  \EE_{\rewardrv \sim \distP_{\action}} \rewardrv$. In addition, there exists a constant $M>0$ such that $| f(\action) - f(\action') | \leq M \cdot \distance ( \action, \action' )$ holds for all $\action, \action' \in \spaceofaction$. Given $\action_t $ and $\dataset_{t-1}$, the feedback $\feedback_t$ is a sample from $\distP_{\action_t}$.
\end{example}

\begin{example}[Dynamic pricing \citep{Den15}]\label{example-pricing}
The set $\spaceofaction \subseteq (0, +\infty)$ consists of feasible prices for a product. Any price $\action$ induces a demand distribution $\distP_{\action}$ over $[0, +\infty)$. The objective value $f(\action)$ is the expected revenue $\action \cdot \EE_{\demandrv \sim \distP_{\action}} \demandrv$.
Given $\action_t $ and $\dataset_{t-1}$, the feedback $\feedback_t$ is a sample from $\distP_{\action_t}$.
\end{example}

\begin{example}[Continuous optimization]\label{example-cvx}
The set $\spaceofaction$ is a subset of a Euclidean space. The objective function $f$ belongs to a known class of continuous functions on $\spaceofaction$, such as convex functions with Lipschitz gradients. Given $\action_t $ and $\dataset_{t-1}$, the feedback $\feedback_t$ may include the function value $f(\action_t)$, the gradient $\nabla f(\action_t)$, the Hessian $\nabla^2 f(\action_t)$, or their noisy versions.
\end{example}

As can be seen from the examples, the historical data only reveals incomplete and noisy information about $f$. To make informed decisions under uncertainty, the agent needs to quantify and update beliefs over time. The Bayesian paradigm offers a coherent framework for this task. We now discuss this approach and the key challenges it faces.

\subsection{Bayesian approaches and their challenges}\label{sec-preliminaries-Bayesian}

Consider a family of stochastic optimization problems $\{ \instance_{\parameter} \}_{\parameter \in \spaceofparameter}$, indexed by a parameter $\parameter$ in a potentially infinite-dimensional space 
$\spaceofparameter$. The parameter $\parameter$ specifies all unknown components of a problem instance, such as the objective function and feedback distributions. Any dataset $\dataset$ defines a likelihood function $\likelihood ( \cdot ; \dataset )$ over $\spaceofparameter$.

The Bayesian paradigm treats the problem \eqref{eqn-problem} as a random instance $\instance_{\parameter}$ whose parameter $\parameter$ is drawn from a prior distribution $\distQ_0$ over the space $\spaceofparameter$ \citep{Tho33,JSW98,SSW15,Fra18}. This prior encodes the agent's initial beliefs, such as smoothness or sparsity of the objective function, based on domain knowledge. After $t$ rounds of interaction, the agent obtains data $\dataset_{t}$ and follows a two-step procedure:
\begin{enumerate}
\item (Belief update) Derive the posterior distribution $\distQ_{t}$ given data $\dataset_{t}$ using Bayes' theorem: 
\begin{align}
\frac{
\rd \distQ_{t} 
}{
\rd \distQ_0
}(\parameter)
=
\frac{
\likelihood ( \parameter ; \dataset_{t} )
}{
\int_{\spaceofparameter}
\likelihood ( \parameter' ; \dataset_{t} ) \distQ_0 ( \rd \parameter' )
},\qquad \parameter \in \spaceofparameter.
\label{eqn-Bayes}
\end{align}
This posterior captures the agent's refined understanding of the problem.
\item (Decision-making) Choose the next decision $\action_{t+1}$ by optimizing a criterion based on $\distQ_{t}$. Popular approaches include expected improvement \citep{Moc74,JSW98}, knowledge gradient \citep{FPD09}, Thompson sampling \citep{Tho33}, Bayesian upper confidence bounds \citep{KCG12}, and information-directed sampling \citep{RVa18}.
\end{enumerate}

While elegant, the Bayesian framework requires specifying a probabilistic model for the \emph{entire} problem instance. This becomes a significant bottleneck when the problem involves rich structural knowledge, as encoding complex constraints through priors can be difficult. In addition, maintaining and sampling from a high-dimensional posterior is computationally expensive. 

We illustrate these challenges using the dynamic pricing problem in Example \ref{example-pricing}. Assume binary demand for simplicity: for any price $\action$, the demand $\feedback$ follows a Bernoulli distribution with parameter $\parameter_{\action} \in [ 0 , 1 ]$. The objective is then $f(\action) = \action \parameter_{\action}$. In the prototypical model, a buyer at time $t$ has a latent valuation $\valuation_t$ drawn from a distribution $\distrho$ over $[0, \infty)$, independently of the history and the posted price $\action_t$. The buyer makes a purchase if and only if $\action_t \leq \valuation_t$, resulting in $\feedback_t = \one ( \action_t \leq \valuation_t )$ and
\begin{align}
\parameter_{\action} = \PP ( \valuation_t \geq \action ) = \distrho ( [\action, +\infty) )
, \qquad \forall \action \in \spaceofaction.
\label{eqn-demand}
\end{align}
It is easily seen that $\distrho$, $\parameter$ and $f$ serve as different parametrizations of the same demand model. 
Working with the $\parameter$-parametrization, we can write the likelihood of data $\dataset_{t}$:
\begin{align}
\likelihood ( \parameter; \dataset_{t} ) = \prod_{i=1}^t \parameter_{\action_i}^{\feedback_i} ( 1 - \parameter_{\action_i} )^{1 - \feedback_i}.
\label{eqn-likelihood-Bernoulli}
\end{align}

A Bayesian algorithm would combine this likelihood with a prior $\distQ_0$ on $\parameter$ to obtain the posterior $\distQ_{t}$ and then select the next price $\action_{t+1}$. This is tractable for simple parametric classes \citep{McL84,FVa10,HKZ12}. On the other hand, nonparametric models offer greater flexibility but run into the obstacle of structural constraints. 
The relation \eqref{eqn-demand} immediately implies that $\parameter$ is non-increasing. Furthermore, if $\distrho$ has a density bounded by $M>0$, then $\parameter$ is $M$-Lipschitz. Below we show in two cases the challenges arising from these constraints.


\paragraph{Case 1. Finite feasible set} Suppose that $\spaceofaction$ only consists of $K$ prices $\price_1 < \cdots < \price_K$. Then, the function $\parameter$ is represented by a vector $\parametervector = (\parameter_{\price_1},\cdots, \parameter_{\price_K} )^{\top}$. The structural constraints confine $\parametervector$ to the following convex set:
\begin{align}
\{ \parametervector \in \RR^K:~ 0 \leq \parameter_K \leq \cdots \leq \parameter_1 \leq 1 \text{ and }
\parameter_{j} - \parameter_{j+1} \leq M ( \price_{j+1} - \price_j ) , ~\forall j \in [K]
 \}.
\label{eqn-parameterspace-pricing}
\end{align}
It is unclear how to design a prior over this set that leads to a tractable posterior. The coupling between the parameters rules out simple product distributions, rendering standard Thompson sampling for Bernoulli bandits \citep{Tho33} inapplicable. Even without the Lipschitz constraint, dealing with the monotonicity requires efforts. For instance, a simple reparametrization maps the monotonicity constraint to the probability simplex, where a Dirichlet prior can be used \citep{Cop07}. However, this approach still requires posterior approximation and does not easily accommodate additional constraints like the Lipschitz condition.

\paragraph{Case 2. Continuous feasible set} Suppose that $\spaceofaction$ is an interval $[ \price_{\min}, \price_{\max} ]$. The structural constraints now define a function class:
\begin{align}
\Big\{ h:~ [ \price_{\min}, \price_{\max} ] \to [0, 1]
~\Big|~ -M \leq h'(\action) \leq 0,~ \forall \action \in [u, v] \Big\}.
\label{eqn-parameterspace-pricing-continuous}
\end{align}
Specifying a tractable prior over this class is not straightforward. Conventional Gaussian processes for Bayesian optimization assign zero probability to a bounded class. Shape-constrained versions are only designed for certain structures \citep{SGF20}.

\vspace{1em}

As the example shows, even simple structural constraints such as monotonicity and Lipschitz continuity pose significant challenges to the standard Bayesian paradigm. We will now introduce a more flexible framework to better incorporate prior knowledge.

\section{A minimalist Bayesian framework}\label{sec-method}

We develop a minimalist Bayesian framework that only specifies a prior for the key component of interest rather than the entire problem instance. This lightweight approach can be easily integrated with structural constraints. We start by modeling the optimum alone, illustrating the idea with a canonical example.

\begin{example}[Multi-armed bandit with Gaussian rewards]\label{example-MAB-Gaussian}
Let $K \in \ZZ_+$, $\spaceofparameter = \RR^K$, and $\sigma > 0$. For any $\parametervector \in \spaceofparameter$, denote by $\instance_{\parametervector}$ the multi-armed bandit problem in Example \ref{example-MAB} with reward distribution $\distP_{j} = N ( \parameter_j , \sigma^2 )$, $\forall j \in [K]$. The likelihood function for a dataset $\dataset_{t}$ is
\begin{align}
\likelihood ( \parametervector ; \dataset_{t} ) = \prod_{i=1}^t \frac{1}{\sqrt{2 \pi } \sigma} 
\exp \bigg(
- \frac{
 ( \parameter_{\action_i} - \feedback_i )^2 
}{2 \sigma^2}
\bigg)
.
\label{eqn-likelihood-MAB}
\end{align}
We use a prior distribution $\distQ_0$ over the decision space $[K]$ to represent our initial belief about which arm is optimal. The statement ``Arm $j$ is optimal'' corresponds to the composite hypothesis $H_{j}:~ \parametervector \in \spaceofparameter_{j}$, where
\begin{align}
	\spaceofparameter_{j} = \{
	\bv \in \spaceofparameter :~ v_j \geq v_k,~\forall k \in [K]
	\}.
	\label{eqn-parameterspace-j}
\end{align}

The likelihood function $\likelihood$ is defined for a point $\parametervector$ rather than a set like $\spaceofparameter_j$. To quantify the evidence for the composite hypothesis $H_j$, we turn to the profile likelihood method \citep{BCo94}. Define the profile likelihood of Arm $j$ as the maximum likelihood achievable by any parameter vector consistent with $H_j$:
\begin{align}
\profilelikelihood ( j ; \dataset_{t} ) = \max_{ \bv \in \spaceofparameter_j }
\likelihood ( \bv ; \dataset_{t} ) .
\label{eqn-profilelikelihood}
\end{align}
This is equivalent to performing constrained maximum likelihood estimation of $\parametervector$ over the set $\spaceofparameter_j$. Finally, we mimick the Bayes' rule to derive a (generalized) posterior $\distQ_{t}$ from the prior $\distQ_0$ and the profile likelihood $\profilelikelihood$:
\begin{align}
\distQ_{t} (j) = \frac{
\profilelikelihood ( j ; \dataset_{t} ) \distQ_0 (j)
}{
\sum_{k=1}^{K} \profilelikelihood ( k ; \dataset_{t} ) \distQ_0 (k)
} , \qquad j \in [K].
\label{eqn-posterior}
\end{align}
It represents our updated belief about which arm is optimal, having integrated the evidence from data.
\end{example}

This approach is efficient, flexible, and general, as highlighted by the following remarks.

\begin{remark}[Computational efficiency]\label{remark-MAB}
	The posterior update is computationally tractable. Let $I_j = \{ i \in [t]:~ \action_i = j \}$ be the set of pulls for Arm $j$, and $\hat{\mean}_j = |I_j|^{-1} \sum_{i \in I_j } \feedback_i$ be its empirical mean. The negative log-likelihood is a weighted sum-of-squares:
	\[
	- \log \likelihood ( \parametervector ; \dataset_{t} ) = 
	\frac{1}{2 \sigma^2 } 
	\sum_{j=1}^{K} |I_j| ( \parameter_j - \hat{\mean}_j )^2 
	+ C,
	\]
	where $C$ is a constant. Denote by $L(\parametervector)$ the first term on the right-hand side. We have
	\[
	\log \profilelikelihood ( j ; \dataset_{t} ) = - \min_{ 
		\parametervector \in \spaceofparameter_j
	}
	L ( \parametervector ) - C .
	\]
	This minimization is a simple quadratic program over a convex polytope, which can be solved efficiently. The generalized posterior $\distQ_t$ is then readily computed from these minimum values:
	\begin{align*}
		\distQ_{t} (j) = \frac{
			e^{ -  \min_{ 
					\parametervector \in \spaceofparameter_j
				}
				L ( \parametervector ) }
			\distQ_0 (j)
		}{
			\sum_{k=1}^{K} 
			e^{ -  \min_{ 
					\parametervector \in \spaceofparameter_k
				}L ( \parametervector ) }
			\distQ_0 (k)
		} .
	\end{align*}
	When $K = 2$, an analytical expression is available. Suppose that $I_1, I_2 \neq \varnothing$ and $\hat{\mean}_1 \geq \hat{\mean}_2$. Let
	\begin{align*}
		&\alpha =  \frac{1}{2\sigma^2} \cdot \frac{ ( \hat\mean_1 - \hat\mean_2 )^2 }{  1/ |I_1| + 1/ |I_2| } .
	\end{align*}
	We have
	\[
	\frac{
		\distQ_t(1) 
	}{
		\distQ_t(2) 
	}
	= e^{\alpha}
	\frac{
		\distQ_0(1) 
	}{
		\distQ_0(2) 
	}
	\qquad\text{and}\qquad
	\distQ_t(1) = 1 - \distQ_t(2) = \frac{
		e^{\alpha} \distQ_0(1)
	}{
		e^{\alpha} \distQ_0(1) + \distQ_0(2)
	}.
	\]
\end{remark}

\begin{remark}[Translation invariance]
The above analysis reveals the translation invariance of \Alg\ with Gaussian reward models: if all the reward distributions are simultaneously shifted by a constant, the generalized posterior for the optimal arm remains unchanged. In contrast, standard Bayesian modeling with a prior for all the mean rewards cannot have such property.
\end{remark}

\begin{remark}[Structured bandits]\label{remark-Lipschitz-bandit}
	Structural constraints on the parameter $\parametervector$ can be seamlessly incorporated by restricting the parameter space $\spaceofparameter$. For instance, adding the Lipschitz condition in Example \ref{example-Lipschitz} leads to
	\begin{align}
		\spaceofparameter = \{ 
		\bv \in \RR^K:~
		|v_i - v_j| \leq M \cdot \distance ( i, j ),~\forall i, j \in [K]
		\}.
		\label{eqn-MAB-Lipschitz}
	\end{align}
	The inference procedure outlined in \eqref{eqn-parameterspace-j}, \eqref{eqn-profilelikelihood} and \eqref{eqn-posterior} remains exactly the same.
\end{remark}

\begin{remark}[Other reward distributions]
	The Gaussian assumption is for illustration. This procedure applies to any reward distribution with a tractable likelihood function, such as Bernoulli or other members of the exponential family.
\end{remark}

The core logic of \Cref{example-MAB-Gaussian} can be abstracted into a general belief-updating algorithm for the location of the optimum. We formalize this in \Cref{alg}.

\begin{algorithm}[h]
	{\bf Input:} A family of stochastic optimization problems $\{ \instance_{\parameter} \}_{\parameter \in \spaceofparameter}$ with decision space $\spaceofaction$ and likelihood function $\likelihood$. A prior distribution $\distQ_0$ over $\spaceofaction$. A dataset $\dataset = \{ ( \action_i, \feedback_i) \}_{i=1}^t$.\\
	{\bf Step 1.} Construct the profile likelihood function
	\begin{align}
		& \profilelikelihood ( \action ; \dataset ) = \sup 
		\Big\{
		\likelihood ( \parameter ; \dataset ) :~
		\parameter \in \spaceofparameter 
		\text{ and }
		f_{\parameter} (\action) = \max_{\action' \in \spaceofaction} f_{\parameter} (\action') 
		\Big\}
		,  \qquad \action \in \spaceofaction,
	\end{align}
	where $f_{\parameter}$ denotes the objective function in the problem instance $\instance_{\parameter}$.
	\\
	{\bf Step 2.} Derive a generalized posterior distribution $\distQ_{t}$ by reweighting the prior $\distQ_0$:
	\begin{align}
		\frac{
			\rd \distQ_{t} 
		}{
			\rd \distQ_0
		} (\action)
		=
		\frac{
			\profilelikelihood ( \action ; \dataset_{t} )
		}{
			\int_{\spaceofaction}
			\profilelikelihood ( \action' ; \dataset_{t} ) \distQ_0 ( \rd \action' )
		}, \qquad  \action \in \spaceofaction.
		\label{eqn-generalized-Bayes}
	\end{align}
	{\bf Output:} $\distQ_{t}$.
	\caption{Minimalist Bayesian inference for the optimum}
	\label{alg}
\end{algorithm}

\Cref{alg} provides a modular inference engine. It can be paired with any decision-making rule that operates on a belief distribution over the optimal action. One natural choice is a minimalist version of Thompson Sampling, presented in \Cref{alg-MINTS}. It is conceptually simpler than standard Thompson Sampling, which starts with a prior for the full parameter $\parameter$ and then draws a decision from the posterior of the optimum in each iterate. Our approach only requires a prior for the optimum.

\begin{algorithm}[h]
	{\bf Input:} A family of stochastic optimization problems $\{ \instance_{\parameter} \}_{\parameter \in \spaceofparameter}$ with decision space $\spaceofaction$ and likelihood function $\likelihood$. A prior distribution $\distQ_0$ over $\spaceofaction$.\\
	Let $\dataset_0 = \varnothing$.\\
	{\bf For $t = 1,2,\cdots$:}\\
	\hspace*{.6cm} Sample a decision $\action_t$ from $\distQ_{t-1}$ and receive feedback $\feedback_t$.\\
	\hspace*{.6cm} Run \Cref{alg} on the updated dataset $\dataset_{t} = \dataset_{t-1} \cup \{ (\action_t, \feedback_t) \}$ to get $\distQ_{t}$.\\
	{\bf Output:} A sequence of decisions $\{ \action_t \}_{t=1}^{\infty}$.
	\caption{MINimalist Thompson Sampling (\Alg)}
	\label{alg-MINTS}
\end{algorithm}

In addition to the optimum, we may also have unknown structural parameters to deal with, such as the noise level $\sigma$ in \Cref{example-MAB-Gaussian} or the Lipschitz constant $M$ in \Cref{example-Lipschitz}. We can use a prior distribution to jointly model the optimum and those parameters. In general, let $\mainparameter$ represent the key components that we wish to model, and $\spaceofmainparameter$ be the space it lives in. For instance, when $\gamma$ encodes the optimum and the Lipshitz constant $M$, we have $\spaceofmainparameter = \spaceofaction \times [0, +\infty)$. Denote by $\spaceofparameter_{\mainparameter}$ the collection of $\parameter$'s whose associated instance $\instance_{\parameter}$ has key components $\mainparameter$. The full parameter space $\spaceofparameter$ is covered by the subsets $\{ \spaceofparameter_{\mainparameter} \}_{\mainparameter \in \spaceofmainparameter}$. 
Let $\distQ_0$ be a prior distribution for the key components. Given data $\dataset_{t}$, the profile likelihood is
\[
\profilelikelihood ( \mainparameter ; \dataset_{t} )
= \sup_{ \parameter \in \spaceofparameter_{\mainparameter} }
\likelihood ( \parameter ; \dataset_{t} ).
\]
Then, we obtain the generalized posterior distribution $\distQ_{t}$ through
\[
\frac{
	\rd \distQ_{t} 
}{
	\rd \distQ_0
} (\mainparameter)
=
\frac{
	\profilelikelihood ( \mainparameter ; \dataset_{t} )
}{
	\int_{\spaceofaction}
	\profilelikelihood ( \mainparameter' ; \dataset_{t} ) \distQ_0 ( \rd \mainparameter' )
}, \qquad  \mainparameter \in \spaceofmainparameter.
\]
\Cref{alg} is a special case where the key component $\mainparameter$ is just the optimum. In the other extreme, when the key component is the whole parameter $\parameter$, each $\spaceofparameter_{\mainparameter}$ becomes a singleton, and we obtain the standard Bayesian procedure.

\section{Examples and new insights}\label{sec-examples}

In this section, we demonstrate the versatility of the minimalist Bayesian framework by applying it to several fundamental problems. We show that it not only handles complex settings like continuum-armed Lipschitz bandits but also provides novel probabilistic interpretations of classical algorithms in convex optimization.

\subsection{Continuum-armed Lipschitz bandit}

Consider a continuum-armed Lipschitz bandit whose decision space $\spaceofaction$ is the $d$-dimensional unit cube $[0, 1]^d$. The parameter space $\spaceofparameter$ consists of all $M$-Lipschitz functions on $\spaceofaction$, where $M>0$ is a known constant.

\paragraph{Noiseless rewards} 
First, assume the feedback $\feedback_t$ is the exact function value, $f(\action_t)$. In this setting, the likelihood $ \likelihood ( h ; \dataset_{t} ) $ is binary: it is 1 if $h$ is consistent with all observations, i.e.~$h(\action_i) = \feedback_i$ for all $i \in [t]$. Consequently, the profile likelihood $\profilelikelihood ( \action ; \dataset_{t} ) $ for an action $\action$ being optimal is also binary: it is 1 if there exists at least one $M$-Lipschitz function that interpolates the data and attains its maximum at $\action$. The following lemma shows that the existence check is a convex feasibility problem. See \Cref{sec-lem-Lipschitz-noiseless-proof} for the proof.

\begin{lemma}\label{lem-Lipschitz-noiseless}
The profile likelihood $\profilelikelihood ( \action; \dataset_{t} )$ is $1$ if the convex polytope
	\[
	S(\action) =
	\{
	(v, v_1,\cdots, v_t)
	:~
	v_i \leq v \leq v_i + M   \| \action - \action_i \|_2 , ~ \forall i  
	\text{ and }
	| v_i - v_j | \leq M   \| \action_i - \action_j \|_2 , ~ \forall i, j  
	\}
	\]
	is non-empty, and $0$ otherwise.
\end{lemma}

Based on that, we have
\begin{align*}
	\frac{
		\rd \distQ_{t} 
	}{
		\rd \distQ_0
	} (\action)
	=
	\frac{
		\one ( S(\action) \neq \varnothing )
	}{
		\int_{\spaceofaction}
		\one ( S(\action') \neq \varnothing ) \distQ_0 ( \rd \action' )
	} .
\end{align*}
The result leads to a rejection sampling algorithm for sampling from the generalized posterior $\distQ_{t}$: draw a candidate optimum $\action'$ from $\distQ_0$, and accept it if $S(\action') \neq \varnothing$.

\paragraph{Gaussian rewards} 
Suppose that the feedback $\feedback_t$ is the function value at $\action_t$ contaminated by random noise from $N(0, \sigma^2)$ with known $\sigma > 0$. The likelihood and profile likelihood are
\begin{align*}
	& \likelihood ( h ; \dataset_{t} ) = \prod_{i=1}^t \frac{1}{\sqrt{2 \pi } \sigma} 
	\exp \bigg(
	- \frac{
		[ h(\action_i) - \feedback_i ]^2 
	}{2 \sigma^2}
	\bigg) , \\
	& \profilelikelihood ( \action ; \dataset_t ) = \sup 
	\Big\{
	\likelihood ( h ; \dataset_t ) \Big|~
	h \in \spaceofparameter
	\text{ and }
	h (\action) = \max_{\action' \in [0,1]^d } h (\action') 
	\Big\}.
\end{align*}
The following lemma shows that for any $\action \in \spaceofaction$, $ \profilelikelihood ( \action ; \dataset_t )$ can be computed up to an additive constant by solving a finite-dimensional convex program. The proof is deferred to \Cref{sec-lem-Lipschitz-noisy-proof}.

\begin{lemma}\label{lem-Lipschitz-noisy}
	For any $\action \in [0, 1]^d $, let $S(\action)$ be the set defined in \Cref{lem-Lipschitz-noiseless}, and $V(\action)$ be the optimal value of the following convex program:
\begin{align}
			\min_{ ( v, v_1, \cdots, v_t ) \in S(\action) }
&  \bigg\{
\frac{1}{2 \sigma^2}	\sum_{i=1}^{t}  ( v_i - \feedback_i )^2 
\bigg\} .
		\label{eqn-Lipschitz-cvx}
\end{align}
	We have $- \log \profilelikelihood ( \action ; \dataset_{t} ) = V (\action) + C$ for a constant $C$. 
\end{lemma}

Define $V_{\min} = \min_{\action \in \spaceofaction} V(\action)$, which can be computed by solving \eqref{eqn-Lipschitz-cvx} without the constraints $v_i \leq v \leq v_i + M   \| \action - \action_i \|_2$ for $i \in [t]$. The facts $V_{\min} - V (\action) \leq 0$ and
\begin{align*}
	\frac{
		\rd \distQ_{t} 
	}{
		\rd \distQ_0
	}
	( \action  )
	=
	\frac{
		e^{ V_{\min} - V (\action) }
	}{
		\int_{\spaceofaction }
		e^{ V_{\min} - V (\action') }   \distQ_0 (\rd \action')
	}
\end{align*}
lead to a rejection sampling scheme for drawing from $\distQ_{t}$: sample $\action'$ from $\distQ_0$ and accept it with probability $e^{ V_{\min} - V (\action') }$.

\subsection{Dynamic pricing}\label{sec-example-pricing}

Next, we revisit the dynamic pricing problem discussed in \Cref{sec-preliminaries-Bayesian}. The demand is binary, the full parameter $\parameter$ encodes the purchase probability at each price, and the likelihood function is defined in \eqref{eqn-likelihood-Bernoulli}. 

Suppose the decision space $\spaceofaction$ consists of $K$ feasible prices $\price_1 < \cdots < \price_K$. The parameter space $\spaceofparameter$ is given by \eqref{eqn-parameterspace-pricing}. Let $I_j = \{ i \in [t]:~ \action_i = \price_j \}$ record the times when price $\price_j$ was chosen, $\hat{\mean}_j = |I_j|^{-1} \sum_{i \in I_j } \feedback_i$ be the empirical mean demand, and  $\spaceofparameter_j$ be defined as in \eqref{eqn-parameterspace-j}. Note that $\spaceofparameter_j$ is a convex polytope and the log-likelihood is a concave function:
\begin{align}
\log \likelihood ( \parametervector ; \dataset_{t} ) = 
\sum_{j=1}^{K} |I_j| \cdot [
\hat{\mean}_j 
\log \parameter_j 
+ (1 - \hat{\mean}_j )
\log ( 1 - \parameter_j )
].
\label{eqn-pricing-log-likelihood}
\end{align}
Then, the profile likelihood $\profilelikelihood ( \price_j , \dataset_{t} ) = \sup_{ \parameter \in \spaceofparameter_{j} } \likelihood ( \parametervector ; \dataset_{t} )$ can be computed through convex optimization, which further yields the generalized posterior $\distQ_{t}$.

We can add more structural constraints based on domain knowledge. As long as the parameter space $\spaceofparameter$ remains convex, the computation of $\distQ_{t}$ is tractable. In contrast, the constraints could render standard Bayesian approaches difficult.

\subsection{First-order convex optimization}

Perhaps surprisingly, our framework provides probabilistic interpretations of classical cutting-plane methods in first-order convex optimization \citep[Section 3.2.8]{Nes18}. Consider minimizing an unknown convex function $f$ over a convex body $\spaceofaction \subseteq \RR^d$, where at each query point $\action_t$, we receive a subgradient $\feedback_t \in \partial f(\action_t)$.
We first offer a minimalist Bayesian perspective for the \emph{center of gravity method}.

\begin{example}[Center of gravity method]\label{example-center-of-gravity}
Denote by $\distQ_0$ the uniform distribution over $\spaceofaction$, which is a natural prior for the unknown optimum. Suppose that for any $t \in \NN$, we apply \Cref{alg} to the dataset $\dataset_{t} = \{ (\action_i, \feedback_i) \}_{i=1}^t$ to derive a generalized posterior $\distQ_{t}$, and use its mean as the next decision $\action_{t+1}$. We have the following results. See \Cref{sec-example-center-of-gravity-proof} for their proof.
\begin{enumerate}
\item $\distQ_{t}$ is the uniform distribution over the set $\{ \action \in \spaceofaction :~
 \feedback_i^{\top} ( \action - \action_i ) \leq 0,~\forall i \in [t]
\}$, which is the intersection of $\spaceofaction$ and $t$ hyperplanes;
\item The procedure is equivalent to the center of gravity method.
\end{enumerate}
\end{example}

Therefore, the center of gravity method can be viewed as a (minimalist) Bayesian approach that uses the (generalized) posterior mean to propose the next query in each iterate. Each query yields a hyperplane that reduces the current search space to one side of it. 
When $t$ is large, $\distQ_{t}$ becomes complex, making it hard to compute the mean. Indeed, the center of gravity method is known to be computationally expensive. 

Meanwhile, Bayesian statistics offers two powerful approaches for handling sophisticated posteriors. One is Markov chain Monte Carlo (MCMC), which designs a Markov chain with the posterior as the limiting distribution; the other is variational Bayes, which approximates the posterior using a family of simple distributions that are easy to sample from.
Interestingly, both ideas lead to well-known optimization algorithms developed as tractable approximations of the center of gravity method. \cite{BVe04} designed an MCMC algorithm for estimating the center of gravity based on a random walk. On the other hand, below we give a variational Bayes interpretation of the celebrated \emph{ellipsoid method}.

\begin{example}[Ellipsoid method]\label{example-ellipsoid}
For any vector $\bc \in \RR^d$ and positive definite matrix $\bA \in \RR^{d\times d}$, define an ellipsoid $E( \bc, \bA ) = \{ \bv \in \RR^d:~
(\bv - \bc)^{\top} \bA (\bv - \bc) \leq 1
 \}$ and denote by $\distQ ( \bc, \bA )$ be the uniform distribution over that. Let $\Phi$ be the parametric distribution family $\{ \distQ ( \bc, \bA ) \}_{\bc \in \RR^d, \bA \succ 0}$.
 
Suppose that $\spaceofaction$ is an ellipsoid $E (\bc_0, \bA_0)$, and let the associated uniform distribution $\distQ_0 = \distQ ( \bc_0, \bA_0 )$ be our prior for the optimum. At any time $t \in \NN$, our belief is characterized by a distribution $\distQ_t = \distQ ( \bc_t, \bA_t ) \in \Phi$. Consider the following updating rule motivated by the assumed density filtering method \citep{Min01}:
\begin{itemize}
	\item use the mean of $\distQ_t$ as the next decision (i.e.~$\action_{t+1} = \bc_t$) and receive feedback $\feedback_{t+1} \in \partial f (\action_{t+1})$;
	\item run \Cref{alg} with prior $\distQ_t$ and data $\{ (\action_{t+1}, \feedback_{t+1}) \}$ to obtain a generalized posterior $\bar\distQ_{t+1}$;
	\item find its best approximation in $\Phi$ that minimizes the forward Kullback-Leibler divergence:
	\[
	\distQ_{t+1} \in 
	\argmin_{ \distQ \in \Phi } D_{\mathrm{KL}} ( \bar\distQ_{t+1} \| \distQ ).
	\]
\end{itemize}
The procedure generates a sequence of uniform distributions over ellipsoids to capture our belief about the optimum. It is equivalent to the ellipsoid method in convex optimization, which has closed-form updates. See \Cref{sec-example-ellipsoid-proof} for a proof.
\end{example}

\section{Numerical experiments}\label{sec-experiments}

Having established the minimalist Bayesian framework, we now demonstrate its efficacy through numerical experiments. The Python code for reproducing the results is available at \url{https://github.com/kw2934/MINTS}. 


Our first example is dynamic pricing with binary demand, as discussed in \Cref{example-pricing}, \Cref{sec-preliminaries-Bayesian} and \Cref{sec-example-pricing}. The decision set $\spaceofaction$ is $\{ j/20:~j \in [19] \}$, with $K = 19$ candidate prices from $0.05$ to $0.95$. The latent valuation distribution $\rho$ is the uniform distribution over the unit interval $[0, 1]$. This determines the expected demand $\parameter_{\action} = 1 - \action$ and expected revenue $f(\action) = \action (1 - \action)$. The objective function is maximized at $\action^* = 1/2$. After $t$ rounds, our data $\dataset_{t} $ consists of chosen prices $\{ \action_i \}_{i=1}^t$ and realized demands $\{ \feedback_i \}_{i=1}^t$. 
We test two versions of \Alg~(\Cref{alg-MINTS}) and two multi-armed bandit algorithms over $T = 5000$ rounds:

\begin{enumerate}
\item \Alg~with Bernoulli likelihood \eqref{eqn-likelihood-Bernoulli}, which uses the correct demand model. The parameter space is the one in \eqref{eqn-parameterspace-pricing} with $M = 1$.

\item \Alg~with Gaussian likelihood \eqref{eqn-likelihood-MAB}, which incorrectly models the conditional demand distribution given price $\action$ as $N( \parameter_{\action} , \sigma^2 )$. We set $\sigma = 1/2$ because the sub-Gaussian variance proxy of the binary demand is bounded by $1/2$ \citep{Hoe94}. The parameter space is the one in \eqref{eqn-parameterspace-pricing} with $M = 1$.

\item Thompson sampling with Gaussian likelihood and prior \citep{RVK18}. It treats the pricing problem as a multi-armed bandit and ignores structural constraints. In the $t$-th round, the reward is the revenue $\action_{t} \feedback_t$ rather than the demand $\feedback_t$. The conditional reward (revenue) distribution given the $j$-th candidate price is modeled as $N( \mean_j , \sigma^2 )$ with $\sigma = 1/2$. The mean parameters $\{ \mean_j \}_{j=1}^K$ are assumed to be independently drawn from a prior distribution $N(0, 1)$.

\item The UCB1 algorithm \citep{ACF02}. Similar to the above, this is also a multi-armed bandit algorithm that directly works with revenues. It constructs optimistic estimates of conditional mean rewards using upper confidence bounds, and chooses the decision with the highest estimate.
\end{enumerate}


We run 100 independent simulations and in each of them, we compute the cumulative regret $ \sum_{i=1}^t 
[
f(\action^*)
-
f ( \action_i ) ]$ over $t \in [T]$ for every algorithm. The results are summarized in \Cref{fig-regrets-pricing}. The left panel shows the regrets of different algorithms. The right panel plots the relative regrets normalized by that of \Alg~with Bernoulli likelihood (skipping the first 20 rounds). In both panels, each solid curve corresponds to the empirical mean value, and the thin shaded areas around that give 95\% confidence bounds on the expected values (empirical mean $\pm$ 1.96 times the standard error). The two \Alg~algorithms significantly outperform bandit algorithms, and the relative advantage grows with time. This highlights the benefit of incorporating structural knowledge. We also note that using a misspecified demand model only has minor impact on the performance of \Alg.

\begin{figure}[t]
	\centering
	\includegraphics[width=0.49\linewidth]{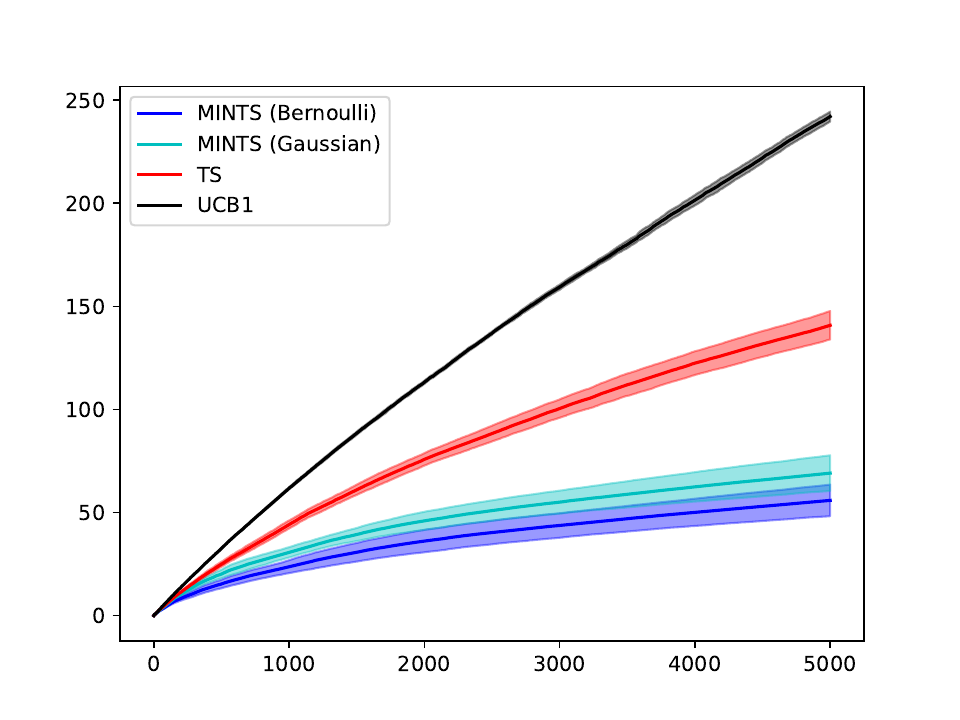}
	\includegraphics[width=0.49\linewidth]{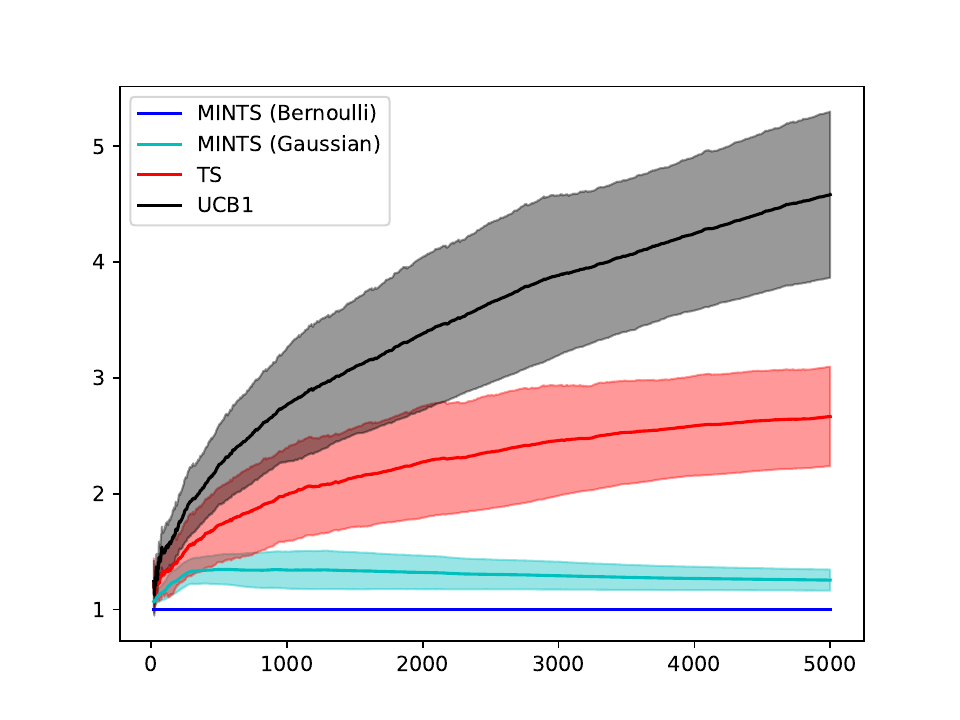}	
	\caption{Dynamic pricing experiments. $x$-axis: time. $y$-axis: regret (left panel) and relative regret (right panel). Blue: \Alg~with Bernoulli likelihood. Cyan: \Alg~with Gaussian likelihood. Red: Thompson sampling. Black: UCB1.}
	\label{fig-regrets-pricing}
\end{figure}


Our second numerical example concerns unstructured multi-armed bandit with Gaussian rewards (\Cref{example-MAB-Gaussian}). There are $K = 5$ arms, and the reward distribution of Arm $j$ is $N(\parameter_j, 1)$ with $\parameter_j = j / 5 $. We compare the following algorithms over $T = 5000$ rounds:

\begin{enumerate}
	\item \Alg~with Gaussian likelihood \eqref{eqn-likelihood-MAB} and $\sigma = 1$. The parameter space is $\RR^K$.
	
	\item Thompson sampling with Gaussian likelihood and prior \citep{RVK18}. The prior distribution is $N(0, 1)$.
	
	\item The UCB1 algorithm \citep{ACF02}.
\end{enumerate}

\Cref{fig-regrets-MAB} summarizes the results over 100 independent simulations, showing the absolute and relative (compared to Thompson sampling) regrets of all algorithms. We see that \Alg~outperforms UCB1 and is comparable to Thompson sampling by a factor of around 2. The latter can be understood as the price of minimalist modeling. Targeting the main parameter while profiling out the nuisance component inevitably loses information. Fortunately, the cost of flexibility is not prohibitive and can be compensated by major gains in structured problems.

\begin{figure}[t]
	\centering
	\includegraphics[width=0.49\linewidth]{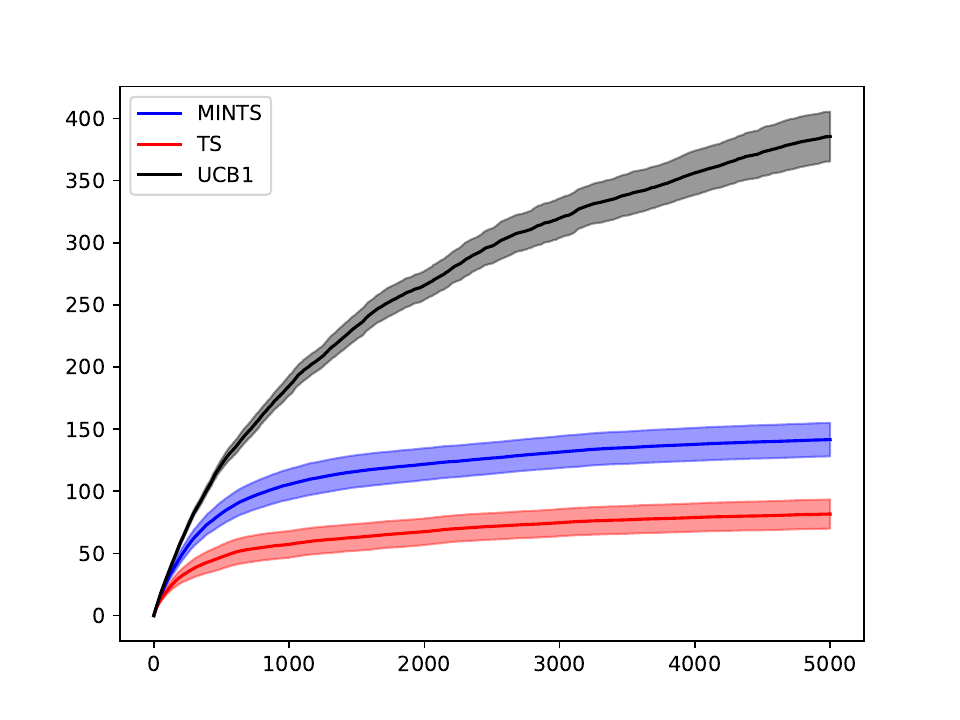}
	\includegraphics[width=0.49\linewidth]{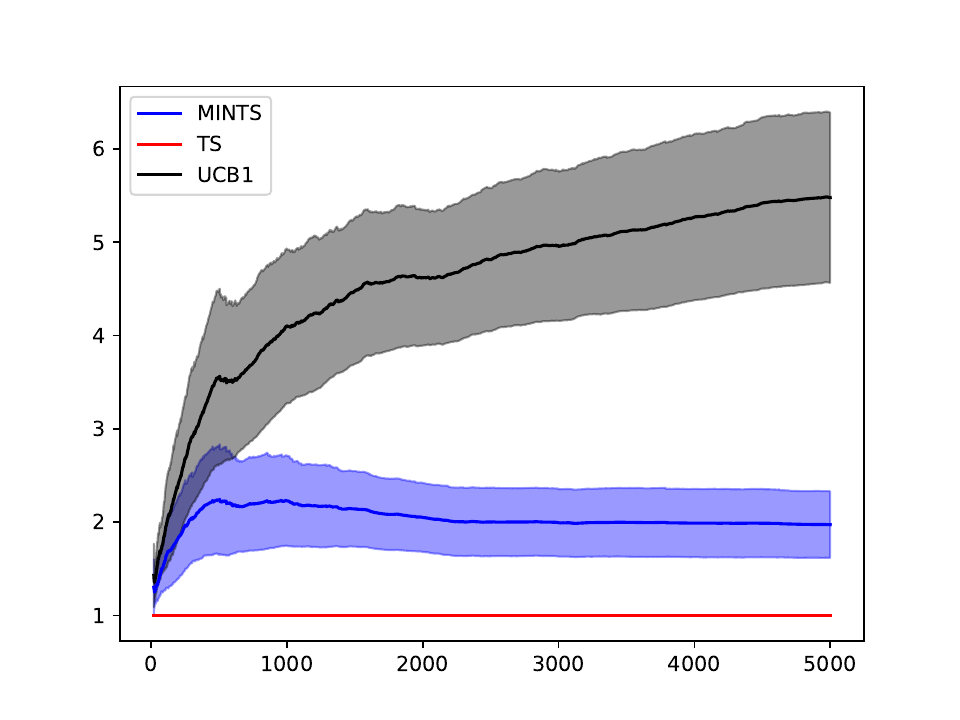}	
\caption{Multi-armed bandit experiments. $x$-axis: time. $y$-axis: regret (left panel) and relative regret (right panel). Blue: \Alg. Red: Thompson sampling. Black: UCB1.}
	\label{fig-regrets-MAB}
\end{figure}

\section{Theoretical analysis for multi-armed bandits}\label{sec-MAB}

In this section, we will analyze \Alg\ (\Cref{alg-MINTS}) for the multi-armed bandit problem in \Cref{example-MAB} and derive near-optimal guarantees. We use $\mean_j$ to refer to the expected reward of Arm $j$, and measure the performance through the \emph{regret} defined below.

\begin{definition}[Regret]
	For any $T \in \ZZ_+$, the regret of a decision sequence $\{ \action_t \}_{t=1}^T$ is
	\[
	\regret(T) =  \sum_{t=1}^T \Big(
	\max_{j \in [K]} \mean_j -
	\mean_{\action_t}
	\Big)  
	.
	\]
\end{definition}


To implement \Alg,\ we model each reward distribution $\distP_j$ as a Gaussian distribution $N ( \mean_j , \sigma^2 )$ with unknown mean $\mean_j$ and known standard deviation $\sigma > 0$. Hence, the bandit problem is modeled as a parametrized instance $ \instance_{\vectorofmeans} $ in \Cref{example-MAB-Gaussian} with unknown $\vectorofmeans \in \RR^K$, and the likelihood function $\likelihood$ is given by \eqref{eqn-likelihood-MAB}. The parametric model is merely a tool for algorithm design; our theoretical guarantees will not be restricted to Gaussian rewards. As we will show shortly, the algorithm performs well for reward distributions satisfying the following light tail condition.

\begin{assumption}[Sub-Gaussian reward]\label{assumption-subg}
The reward distributions $\{ \distP_j \}_{j=1}^K$ are 1-sub-Gaussian:
\begin{align*}
\EE_{\reward \sim \distP_j} e^{ \lambda (\reward - \mean_j) } \leq e^{ \lambda^2 / 2 } , \qquad \forall 
 \lambda \in \RR.
\end{align*}
\end{assumption}

Assumption \ref{assumption-subg} is standard for bandit studies \citep{LSz20}. It holds for many common distributions with sufficiently fast tail decay, including any Gaussian distribution with variance bounded by 1, or distributions supported on an interval of width 2 \citep{Hoe94}. For sub-Gaussian distributions with general variance proxies, we can reduce to this case by rescaling.

We present a regret bound with explicit dependence on the horizon $T$, number of arms $K$, and the sub-optimality gaps of arms. The proof is deferred to \Cref{sec-cor-regret-proof}.

\begin{theorem}[Regret bound]\label{cor-regret}
For the multi-armed bandit in \Cref{example-MAB}, run \Alg\ (\Cref{alg-MINTS}) with a uniform prior over the $K$ arms and the Gaussian likelihood \eqref{eqn-likelihood-MAB} with $\sigma > 1$. Define $\Delta_j = \max_{k \in [K]} \mean_k - \mean_j$ for $j \in [K]$. Under Assumption \ref{assumption-subg}, there exists a constant $C$ determined by $\sigma$ such that
	\begin{align*}
		& \EE[ \regret(T) ]
		\leq C 
		\bigg(
		\min \bigg\{
		\sum_{j:~ \Delta_j > 0} 
		\frac{\log T }{\Delta_j }  ,
		\sqrt{ K T \log K } 
		\bigg\}
		+ \sum_{j=1}^{K} \Delta_j
		\bigg).
	\end{align*}
\end{theorem}

This result demonstrates the near-optimality of \Alg.\ The sum $ \sum_{j=1}^{K} \Delta_j$ stems from the fact that each arm must be pulled at least once. When the problem instance is fixed and $T$ is sufficiently large, the problem-dependent bound $	\sum_{j:~ \Delta_j > 0} \Delta_j^{-1} \log T  $ matches the lower bound for Gaussian bandits in \cite{GMS19} up to a constant factor. 
For any fixed $T$, the problem-independent bound $\sqrt{TK\log K}$ matches the regret bound for Thompson sampling using Gaussian priors and likelihood \citep{AGo17}, achieving the minimax lower bound in \cite{BCe12} up to a $\sqrt{\log K}$ factor.

\Cref{cor-regret} is a corollary of the more refined result below. See \Cref{sec-thm-regret-proof} for the proof.

\begin{theorem}[Regret bound]\label{thm-regret}
	Under the setup in \Cref{cor-regret}, there exists a constant $C$ determined by $\sigma$ such that 
	\begin{align*}
		\EE [\regret(T)] \leq C 
		\inf_{ \delta \geq 0 }
		\bigg\{
		\sum_{j:~ \Delta_j > \delta}
		\bigg(
		\frac{ \log ( \max\{ T \Delta_j^2 , e \} )  }{\Delta_j } 
		+  \Delta_j  
		\bigg)
		+ T \max_{j:~\Delta_j \leq \delta} \Delta_j
		\bigg\} .
	\end{align*}
\end{theorem}
The regret bound has the same order as that in \cite{AOr10}, achieved by a carefully designed upper confidence bound algorithm with arm elimination. Our algorithm is simpler.

\section{Discussion}\label{sec-discussion}

We introduced a minimalist Bayesian framework for stochastic optimization that only requires a prior for the component of interest and handles nuisance parameters via profile likelihood. The lightweight modeling makes it easy to incorporate structural constraints on problem parameters, opening several promising avenues for future research. First, designing scalable algorithms for sampling from the generalized posterior is critical for handling continuous or high-dimensional spaces. Second, developing more sophisticated acquisition rules beyond simple posterior sampling could further improve performance. Beyond these refinements, extending the minimalist principle to contextual bandits and reinforcement learning presents an exciting frontier. Finally, a crucial theoretical task will be to accompany these new algorithms with rigorous guarantees.

\section*{Acknowledgement}
The author thanks Yeon-Koo Che, Yaqi Duan and Chengpiao Huang for helpful discussions. This research is supported by NSF grants DMS-2210907 and DMS-2515679.


\appendix

\section{Proofs of \Cref{sec-examples}}

\subsection{Proof of \Cref{lem-Lipschitz-noiseless}}\label{sec-lem-Lipschitz-noiseless-proof}

Choose any $\action$ such that $ \profilelikelihood ( \action ; \dataset_{t} ) = 1$. By definition, there exists $h \in \spaceofparameter$ such that $h(\action) = \max_{\action' \in \spaceofaction } h(\action')$ and $h(\action_i) = \feedback_i$ for all $i \in [t]$. It is easily seen that $( h(\action), h(\action_1),\cdots, h(\action_t) ) \in S(\action)$ and thus $S(\action) \neq \varnothing$.

Now, suppose that $\action$ makes $S(\action) \neq \varnothing$. We need to show that $ \profilelikelihood ( \action ; \dataset_{t} ) = 1$. Choose any $(v, v_1,\cdots, v_t) \in S(\action)$. The Kirszbraun theorem \citep{Kir34} guarantees the existence of a function $g:~\spaceofaction \to \RR$ that is $M$-Lipschitz and satisfies $g(\action) = v$, $g(\action_1) = v_1$, $\cdots$, $g(\action_t) = v_t$. Define $h(\cdot) = \max \{ g(\cdot), v \}$. Then, $h$ remains $M$-Lipschitz. It is maximized at $\action$ and agrees with $g$ on $\action,\action_1,\cdots, \action_t$. Therefore, $ \profilelikelihood ( \action ; \dataset_{t} ) = 1$.

\subsection{Proof of \Cref{lem-Lipschitz-noisy}}\label{sec-lem-Lipschitz-noisy-proof}

Define $L(v_1,\cdots,v_n) = \frac{
	1
}{2 \sigma^2}
\sum_{i=1}^{t}
( v_i - \feedback_i )^2 $. It is easily seen that 
\[
- \log \likelihood ( h ; \dataset_{t} ) =  L( h(\action_1), \cdots, h(\action_t) ) + C, \qquad\forall h \in \spaceofparameter
\]
holds with a constant $C$. For any $\varepsilon > 0$, there exists $g \in \spaceofparameter$ such that $g(\action) = \max_{\action' \in \spaceofaction } g(\action')$, $g(\action_i) = \feedback_i$ for all $ i \in [t]$, and
\[
- \log \likelihood ( g ; \dataset_{t} ) \leq - \log \profilelikelihood ( \action ; \dataset_{t} ) + \varepsilon.
\]
Note that $( g(\action), g(\action_1),\cdots, g(\action_t) )$ is feasible for the program \eqref{eqn-Lipschitz-cvx}. Hence,
\[
V(\action) \leq L ( g( \action_1),\cdots, g(\action_t) )
= - \log \likelihood ( g ; \dataset_{t} ) - C 
\leq - \log \profilelikelihood ( \action ; \dataset_{t} ) - C + \varepsilon.
\]
We get $- \log \profilelikelihood ( \action ; \dataset_{t} )  \geq V(\action) + C $.

It remains to prove the other direction, choose any optimal solution $( v, v_1, \cdots, v_t )$ to \eqref{eqn-Lipschitz-cvx}. Similar to the proof of \Cref{lem-Lipschitz-noiseless}, we can construct $h \in \spaceofparameter$ that is maximized at $\action$ and satisfies $h(\action) = v$, $h(\action_1) = v_1$, $\cdots$, $h(\action_t) = v_t$. Then,
\[
V(\action) =  L( h(\action_1), \cdots, h(\action_t) ) =
- \log \likelihood ( h ; \dataset_{t} ) - C
\geq  - \log \profilelikelihood ( \action ; \dataset_{t} ) - C .
\]
We get $- \log \profilelikelihood ( \action ; \dataset_{t} )  \leq V(\action) + C $.

\subsection{Proof of the claims in \Cref{example-center-of-gravity}}\label{sec-example-center-of-gravity-proof}

Since the feedback is noiseless, the likelihood and profile likelihood are binary-valued:
\begin{align*}
	& \likelihood ( f ; \dataset_{t} ) = \one (
	\feedback_i \in \partial f(\action_i),~\forall i \in [t]
	) ,\\
	& \profilelikelihood ( \action ; \dataset_{t} ) = \one \Big(
	\exists f \in \functionclass, 
	\text{ s.t. }
	f(\action) = \max_{\action' \in \spaceofaction } f(\action')
	\text{ and }
	\feedback_i \in \partial f(\action_i),~\forall i \in [t]
	\Big).
\end{align*}
Hence, the generalized posterior $\distQ_{t}$ is the uniform distribution over the set
\begin{align*}
	\spaceofaction_t = \Big\{
	\action \in \spaceofaction:~
	\exists f \in \functionclass
	\text{ s.t. }
	f(\action) = \min_{\action' \in \spaceofaction} f(\action')
	\text{ and }
	\feedback_i \in \partial f(\action_i),~\forall i \in [t]
	\Big\}.
\end{align*}
On the other hand, let $\cS= \{ \action \in \spaceofaction :~
 \feedback_i^{\top} ( \action - \action_i ) \leq 0,~\forall i \in [t]
\}$. We only need to prove that $\spaceofaction_t = \cS$.

The first step is to show $\spaceofaction_t \subseteq \cS$. For any $\action \in \spaceofaction_t$, there exists a convex function $f$ on $\spaceofaction$ that attains its minimum value at $\action$ and satisfies $\feedback_i \in \partial f(\action_i)$ for all $i \in [t]$. 
The optimality of $\action$ and convexity of $f$ imply that
\[
0 \geq f(\action) - f(\action_i) \geq \feedback_i^{\top} ( \action - \action_i ) , \qquad \forall i \in [t]
\]
and thus $\action \in \cS$. Consequently, $\spaceofaction_t \subseteq \cS$. 

It remains to prove $\cS \subseteq \spaceofaction_t$. Choose any $\action \in \cS$ and define
\[
f(\action') = \max_{i \in [t]} 
[ \feedback_i^{\top} ( \action' - \action_i ) ]_+, \qquad \action' \in \spaceofaction.
\]
Here, $(\cdot)_+ = \max \{ \cdot, 0 \}$ denotes the positive part of a real number. The function $f$ is clearly convex and nonnegative. Meanwhile, the assumption $\action \in \cS$ forces $f(\action) = 0$, which further implies that $\action$ is an optimum of $f$. Therefore, $\action \in \spaceofaction_t$. We get $\cS \subseteq \spaceofaction_t$.

\subsection{Proof of the claim in \Cref{example-ellipsoid}}\label{sec-example-ellipsoid-proof}

By applying the first result in \Cref{example-center-of-gravity} to the new procedure, we see that $\bar\distQ_{t+1}$ is the uniform distribution over a half ellipsoid $\{ \action \in E ( \bc_t, \bA_t ) :~ \feedback_{t+1}^{\top} (\action - \action_{t+1}) \leq 0 \}$. Let $\bar{E}_{t+1}$ denote this region. 

Choose any $\distQ = \distQ ( \bc, \bA ) \in \Phi$. To make $D_{\mathrm{KL}} ( \bar\distQ_{t+1} \| \distQ )$ finite, we must have $\bar\distQ_{t+1} \ll \distQ$ and thus $\bar{E}_{t+1} \subseteq E(\bc  , \bA )$. Then, we have
\[
\frac{\rd \bar\distQ_{t+1}}{
	\rd \distQ } (\action) = 
\frac{ \mathrm{Vol} [  E(\bc  , \bA ) ] }{
	\mathrm{Vol} (\bar{E}_{t+1})
}
\one ( \action \in \bar{E}_{t+1} ),
\]
where $\mathrm{Vol}(\cdot)$ denotes the volume of a region. Consequently,
\begin{align*}
	D_{\mathrm{KL}} ( \bar\distQ_{t+1} \| \distQ ) = \EE_{\action \sim \bar\distQ_{t+1}}
	\bigg[
	\log \bigg(
	\frac{\rd \bar\distQ_{t+1}}{
		\rd \distQ } (\action) 
	\bigg)
	\bigg]
	= \frac{1}{  \mathrm{Vol} (\bar{E}_{t+1}) } \cdot \log \bigg(
	\frac{ \mathrm{Vol} [  E(\bc  , \bA ) ] }{
		\mathrm{Vol} (\bar{E}_{t+1})
	}
	\bigg).
\end{align*}
This is monotonically increasing in $\mathrm{Vol} [  E(\bc  , \bA ) ]$. As $\distQ_{t+1}$ minimizes the divergence, $E(\bc_{t+1} , \bA_{t+1})$ must have the smallest volume among all ellipsoids covering $\bar{E}_{t+1}$. This is precisely the updating rule for the ellipsoid method in convex optimization.

\section{Proofs of \Cref{sec-MAB}}

\subsection{Proof of \Cref{thm-regret}}\label{sec-thm-regret-proof}

We first decompose the regret into the contributions of individual arms:
\begin{align}
	\EE [\regret(T) ]
	=  \sum_{t=1}^T  \sum_{j=1}^K  \Big(
	\max_{k \in [K]} \mean_k - \mean_j 
	\Big)  \PP (\action_t = j) 
	=  \sum_{j=1}^K \Delta_j \bigg( \sum_{t=1}^{T}  \PP ( \action_t = j ) \bigg)  .
	\label{eqn-thm-regret-0}
\end{align}
Next, we invoke a lemma on the expected number of pulls of any sub-optimal arm. The proof borrows ideas from the analysis of Thompson sampling by \cite{AGo17} and is deferred to \Cref{sec-lem-regret-j-proof}.

\begin{lemma}\label{lem-regret-j}
	There exists a universal constant $C_0 > 0$ such that if $\Delta_j > 0$, then
	\begin{align*}
		& 
		\sum_{t=1}^{T}  \PP ( \action_t = j ) \leq C_0
		\bigg(
		\frac{ \sigma^2}{1 - \sigma^{-2}} \cdot  \frac{ \log ( \max\{ T \Delta_j^2 , e \} ) }{ \Delta_j^2 } 
		+ \frac{1}{ \sqrt{1 - \sigma^{-2}} }
		\bigg)
		.
	\end{align*}
\end{lemma}

Choose any $\delta \geq 0$. Then,
\begin{align*}
	\sum_{j:~  \Delta_j \leq \delta} \Delta_j \bigg( \sum_{t=1}^{T}  \PP ( \action_t = j ) \bigg) 
	\leq 
	\max_{j:~\Delta_j \leq \delta} \Delta_j \cdot
	\sum_{t=1}^{T} \sum_{j = 1}^K \PP ( \action_t = j )
	= T \max_{j:~\Delta_j \leq \delta} \Delta_j.
\end{align*}
When $ \Delta_j >  \delta$, we use \Cref{lem-regret-j} to obtain that
\begin{align*}
	\Delta_j  \sum_{t=1}^{T}  \PP ( \action_t = j ) 
	& \lesssim
	\frac{   \log ( \max\{ T \Delta_j^2 , e \} )  }{\Delta_j } 
	+  \Delta_j  ,
\end{align*}
where $\lesssim$ hides a constant factor determined by $\sigma$. Plugging these estimates into \eqref{eqn-thm-regret-0} finishes the proof.

\subsection{Proof of \Cref{cor-regret}}\label{sec-cor-regret-proof}

The result is trivial when $K = 1$ or $T = 1$. From now on, we assume $K \geq 2$, $T \geq 2$ and use $\lesssim$ to hide constant factors determined by $\sigma$.

By taking $\delta = 0$ in the regret bound in \Cref{thm-regret}, we obtain that
\[
\EE[ \regret(T) ]
\lesssim
\sum_{j:~ \Delta_j > 0}
\frac{ \log (  
	\max \{ T \Delta_j^2, e \} 	
	) }{\Delta_j } 
+ 
\sum_{j=1}^K
\Delta_j  .
\]
Note that 
\begin{align*}
	\frac{ \log (  
		\max \{ T \Delta_j^2, e \} 	
		) }{\Delta_j } 
	& = \frac{ 
		\max \{ \log T + 2 \log \Delta_j, 1 \}
	}{\Delta_j } 
	\leq\frac{1 + \log T}{\Delta_j} + \frac{ 2 \log (1 + \Delta_j ) }{\Delta_j} \\
	&
	\overset{(\mathrm{i})}{\leq}
	\frac{1 + \log T}{\Delta_j} + 2
	\overset{(\mathrm{ii})}{\leq}
	\frac{1 + \log T}{\Delta_j} + \bigg(
	\Delta_j + \frac{1}{\Delta_j}
	\bigg)
	\lesssim \frac{ \log T}{\Delta_j} + \Delta_j
	,
\end{align*}
where $(\mathrm{i})$ and $(\mathrm{ii})$ follow from elementary inequalities $\log (1+z) \leq z$ and $z + 1/z \geq 2$ for all $z > 0$. As a result,
\begin{align}
	\EE[ \regret(T) ]
	\lesssim
	\sum_{j:~ \Delta_j > 0}
	\frac{ \log T}{\Delta_j} + \sum_{j=1}^K \Delta_j.
	\label{eqn-cor-1}
\end{align}

On the other hand, \Cref{thm-regret} implies that
\[
\EE [\regret(T)] \lesssim
\inf_{\delta \geq 0}
\bigg\{
\sum_{j:~ \Delta_j > \delta}
\frac{ \log (  
	\max \{ T \Delta_j^2, e \} 	
	) }{\delta}
+  T \delta
\bigg\}
+ \sum_{j=1}^K	 \Delta_j  .
\]
Choose any $\delta \geq e / \sqrt{T}$. When $\Delta_j > \delta$, we have $T \Delta_j^2 \geq e^2 > e$ and
\begin{align*}
	\frac{
		\log (  
		\max \{ T \Delta_j^2, e \} 	
		)	
	}{\Delta_j } =
	\frac{ 2 \log ( \sqrt{T} \Delta_j )}{\Delta_j}=
	2 \sqrt{T} \cdot \frac{ \log ( \sqrt{T} \Delta_j ) }{
		\sqrt{T} \Delta_j 
	} 
	\leq 2 \sqrt{T} \cdot  \frac{ \log ( \sqrt{T} \delta ) }{
		\sqrt{T} \delta 
	} =  \frac{ 2 \log ( \sqrt{T} \delta ) }{
		\delta 
	} ,
\end{align*}
where the inequality follows from the facts that $\sqrt{T} \Delta_j \geq \sqrt{T} \delta \geq e$ and $z \mapsto z^{-1} \log z$ is decreasing on $[e, +\infty)$. Consequently,
\[
\EE [\regret(T)]
\lesssim
\inf_{\delta \geq e / \sqrt{T}}
\bigg\{
\frac{ K \log ( \sqrt{T} \delta) }{\delta}
+  T \delta
\bigg\}
+ \sum_{j=1}^K	 \Delta_j  .
\]
Taking $\delta = e \sqrt{ T^{-1} K \log K }$, we get
\begin{align}
	\EE [\regret(T)]
	\lesssim 	\sqrt{ T K \log K } + \sum_{j=1}^K	 \Delta_j  
	.
	\label{eqn-cor-2}	
\end{align}
The proof is completed by combining \eqref{eqn-cor-1} and \eqref{eqn-cor-2}.

\subsection{Proof of \Cref{lem-regret-j}}\label{sec-lem-regret-j-proof}

\subsubsection{Preparations}

Following the convention, we represent the reward by $\reward_t$ rather than $\feedback_t$. We now introduce some key quantities for tracking the iterates.

\begin{definition}
	For any $j \in[K]$ and $t \in \ZZ_+$, denote by $\pullindexset_j (t) =  \{ i \in [t-1]:~ \action_i = j \} $ the set of pulls for Arm $j$ in the first $(t-1)$ rounds, and $\pullcount_j (t) = | \pullindexset_j (t) |$ the number of pulls. When $\pullcount_j(t) \geq 1$, let
	\[
	\hat\mean_j (t) = \frac{1}{\pullcount_j(t)} \sum_{
		i \in \pullindexset_j(t)
	} \reward_i
	\]
	be the empirical mean reward of Arm $j$. When $\pullcount_j(t) = 0$, let $\hat\mean_j (t) = 0$.
\end{definition}

\begin{definition}\label{defn-filtration}
	Denote by $\tau_{j,k}$ the time of the $k$-th pull of Arm $j$. Let $\xi_{j, k} = \frac{1}{k} \sum_{i=1}^{k} \reward_{\tau_{j,i}} $ be the average reward over the first $k$ pulls of Arm $j$. Let $\history_t$ be the $\sigma$-field generated by the data $\dataset_{t}$. 
\end{definition}

Choose any $M > 0$ and $j \in [K]$ such that $\Delta_j > 0$. Define $u_j = \mean_j + \Delta_j / 3$, $v_j = \mean_j + 2 \Delta_j / 3$, and
\begin{align*}
	& \cJ_1 =  \sum_{t=1}^T \PP [ \action_t = j, \pullcount_j(t) \leq M ], \\
	& \cJ_2 = \sum_{t=1}^T \PP [ \action_t = j, \hat{\mean}_j(t) \geq u_j ] , \\
	& \cJ_3 = \sum_{t=1}^T \PP [ \action_t = j,\pullcount_j(t) > M, \hat{\mean}_j(t) < u_j ] .
\end{align*}
We have a decomposition
\begin{align}
	\sum_{t=1}^{T}  \PP ( \action_t = j ) \leq \cJ_1 + \cJ_2 + \cJ_3 .
	\label{eqn-J-0}
\end{align}
By definition,
\begin{align}
& 	\cJ_1 
	= \sum_{t=1}^T \EE \Big( \one [ \action_t = j, \pullcount_j(t) \leq M ] \Big)
	= \EE \bigg( 
	\sum_{t=1}^T \one [ \action_t = j, \pullcount_j(t) \leq M ]
	\bigg) \leq M + 1
	,	\label{eqn-J-3} 
	\\ 
& \cJ_2 \leq \sum_{k = 1 }^{\infty} \PP ( \xi_{j, k} \geq u_j )
	= \sum_{k = 1 }^{\infty} 
	 \PP ( \xi_{j, k} - \mean_j \geq \Delta_j / 3 )  .
		\label{eqn-J-2-1} 
\end{align}
We invoke useful concentration bounds on the difference between the empirical average reward $\xi_{j, k}$ and the expectation $\mean_j$.

\begin{lemma}\label{lem-tail}
	Under Assumption \ref{assumption-subg}, we have 
	\begin{align*}
		& \PP  (
		\xi_{j, k} - \mean_j \geq t
		) \leq  e^{-kt^2 / 2} , \qquad \forall t \geq 0 , \\
		& \PP  (
		\xi_{j, k} - \mean_j \leq - t
		) \leq  e^{-kt^2 / 2} , \qquad \forall t \geq 0 , \\
		& \EE e^{ \lambda k ( \xi_{j, k} - \mean_j )^2 /2} \leq \frac{1}{\sqrt{1 - \lambda}}, \qquad \forall \lambda \in [0, 1).
	\end{align*}
\end{lemma}

\begin{proof}[\bf Proof of \Cref{lem-tail}]
	Note that $\{ \xi_{j, k} - \mean_j \}_{k=1}^{\infty}$ is a martingale difference sequence with respect to the filtration $\{ \history_{\tau_{j,k}} \}_{k=1}^{\infty}$. Theorem 2.19 in \cite{Wai19} yields the desired tail bounds on $\xi_{j, k} - \mean_j $, together with the fact that $\xi_{j, k}$ is $k^{-1}$-sub-Gaussian. The proof is then completed by applying Theorem 2.6 in \cite{Wai19}.
\end{proof}


By \eqref{eqn-J-2-1} and \Cref{lem-tail}, 
\begin{align*}
\cJ_2 
\leq  \sum_{k = 0 }^{\infty} e^{-k \Delta_j^2 / 18} = \frac{1}{1 -  e^{-\Delta_j^2 / 18} } .
\end{align*}
For any $z > 0$, we have $e^z \geq 1 + z$ and thus $e^{-z} \leq  (1+z)^{-1}$. Then,
\begin{align}
	\cJ_2 
	\leq \frac{1}{1 -  
		 (1+\Delta_j^2 / 18)^{-1}
}
= \frac{18}{\Delta_j^2} + 1.
	\label{eqn-J-2}
\end{align}
It remains to bound $\cJ_3$. 

\subsubsection{Bounding $\cJ_3$}

Without loss of generality, we assume $\mean_1 = \max_{k \in [K]} \mean_k$ throughout the proof. As a result, $\Delta_j = \mean_1 - \mean_j$. Let $c\in (0, 1)$ be a constant to be determined, and $M'  = (1-c) M$. We have
\begin{align*}
	& \PP [ \action_t = j, \pullcount_j(t) > M, \hat{\mean}_j(t) < u_j ]  \\
	& = \PP [ \action_t = j, \pullcount_j(t) > M, \hat{\mean}_j(t) < u_j, \pullcount_1(t) > M' ,
	\hat{\mean}_1 (t) > v_j] \\
	&\quad + \PP [ \action_t = j, \pullcount_j(t) > M, \hat{\mean}_j(t) < u_j, 
	\pullcount_1(t) > M' , \hat{\mean}_1 (t) \leq v_j  ] \\
	&\quad + \PP [ \action_t = j, \pullcount_j(t) > M, \hat{\mean}_j(t) < u_j, 
	1 \leq \pullcount_1(t) < M', \hat{\mean}_1 (t) > \hat{\mean}_j (t) ] \\
	&\quad + \PP [ \action_t = j, \pullcount_j(t) > M, \hat{\mean}_j(t) < u_j, 
	1 \leq \pullcount_1(t) < M', \hat{\mean}_1 (t) \leq \hat{\mean}_j (t) ] \\
	&\quad + \PP [ \action_t = j, \pullcount_j(t) > M, \hat{\mean}_j(t) < u_j, \pullcount_1(t) = 0 ] .
\end{align*}
Denote by $\cE_{1, t}$, $\cE_{2, t}$, $\cE_{3, t}$, $\cE_{4, t}$ and $\cE_{5, t}$ the five summands on the right-hand side. We have 
\begin{align}
	\cJ_3 \leq \sum_{t=1}^{T} ( \cE_{1, t} + \cE_{2, t} + \cE_{3, t} + \cE_{4, t} + \cE_{5, t} )
	\label{eqn-J-1-0}
\end{align}
We will control the $\cE_{j, t}$'s individually. The following fact will come in handy: for any $\history_{t-1}$-measurable event $\eventA$,
\begin{align}
\PP ( \{ \action_t = j \} \cap \eventA )
& = \EE [ \PP (  \{ \action_t = j \} \cap \eventA | \history_{t-1} ) ] 
= \EE [ \PP ( \action_t = j  | \history_{t-1} ) \cdot \one (\eventA) ] 
= \EE [ \distQ_{t-1}(j) \one (\eventA) ] .
	\label{eqn-conditioning-1}
\end{align}
We also need to characterize the generalized posterior $\distQ_{t}$. Remark \ref{remark-MAB} implies that
\[
\distQ_{t} (j) 
= \frac{ e^{- 	\Lambda ( j , \dataset_{t} ) } \distQ_0 (j) }{ 
	\sum_{k=1}^{K} e^{ - 
		\Lambda ( k , \dataset_{t} ) }  \distQ_0 (k)
} ,
\]
where
\begin{align}
	\Lambda ( j , \dataset_{t} ) = 
	\min_{ \btheta \in \spaceofparameter_j } \bigg\{
	\frac{1}{2 \sigma^2 } \sum_{k=1}^{K} 
	\pullcount_k (t+1) [ \hat\mean_k (t+1) - \parameter_k ]^2 	
	\bigg\}.
		\label{eqn-Lambda}
\end{align}
Then, we have
\begin{align}
	\frac{\distQ_{t} (j) }{
		\distQ_{t} (i) 
	}
	= e^{\Lambda(i, \dataset_{t}) - \Lambda(j, \dataset_{t})} 
		\frac{\distQ_0 (j) }{
				\distQ_0 (i) 
			}
	.
	\label{eqn-ratio}
\end{align}

We invoke some useful estimates for $\Lambda$, whose proof is deferred to \Cref{sec-lem-ratio-proof}.

\begin{lemma}\label{lem-ratio}
	Suppose that $i, j \in [K]$ and $\hat{\mean}_i(t) \geq \hat{\mean}_j(t)$.
	\begin{enumerate}
		\item\label{lem-ratio-1} 
		We have
		\[
		\Lambda (j, \dataset_{t-1})
		\geq \frac{1}{2 \sigma^2} \cdot \frac{ [ \hat\mean_{i} (t)  - \hat\mean_j (t)  ]^2 }{  1/ \pullcount_{i} (t)  + 1/ \pullcount_j  (t) } .
		\]
		
		\item\label{lem-ratio-2} If $\pullcount_j(t) \geq \pullcount_i(t)$, then $ \Lambda ( i, \dataset_{t-1}) - \Lambda (j, \dataset_{t-1}) \leq 0 $.
		
		\item\label{lem-ratio-3} If $\pullcount_j(t) < \pullcount_i(t)$, then
		\[
		\Lambda ( i, \dataset_{t-1}) - \Lambda (j, \dataset_{t-1}) 
		\geq 
		- \frac{1}{2 \sigma^2}
		\cdot
		\frac{
			(  \hat\mean_i  - \hat\mean_j )^2
		}{
			1/\pullcount_j   - 1/ \pullcount_i 
		}
		.
		\]
	\end{enumerate}
\end{lemma}

We are now in a position to tackle the summands $\{ \cE_{j, t} \}_{j=1}^5$ in \eqref{eqn-J-1-0}.

\paragraph{Bounding $\cE_{1, t}$.}
Let $\eventA$ be the event $\{  \pullcount_j(t) > M, \hat{\mean}_j(t) < u_j, \pullcount_1(t) > M' ,
\hat{\mean}_1 (t) > v_j \}$. Choose $\hat\action_t \in \argmax_{j \in [K]} \hat{\mean}_j(t)$. We have $\Lambda (
\hat\action_t
, \dataset_{t-1}) = 0$. Then, the relation \eqref{eqn-ratio} and the uniformity of $\distQ_0$ yield
\[
\distQ_{t-1}(j) 
= e^{ \Lambda (
	\hat\action_t
	, \dataset_{t-1}) - \Lambda (j, \dataset_{t-1})}
\distQ_{t-1}(\hat{\action}_t) 
\leq e^{ -\Lambda (j, \dataset_{t-1})}.
\]
Under $\eventA$, Part \ref{lem-ratio-1} of \Cref{lem-ratio} implies that
\begin{align*}
\distQ_{t-1}(j) 
& 
 \leq \exp\bigg(
- \frac{1}{2 \sigma^2}
\cdot
\frac{ (v_j - u_j)^2 }{  1/ M + 1/ M' } 
\bigg)
= \exp\bigg(
- \frac{M'
\Delta_j^2 
}{ 36 \sigma^2}
\bigg).
\end{align*}
By \eqref{eqn-conditioning-1},
\begin{align}
\cE_{1, t} 
= 
 \EE [
\distQ_{t-1}(j)  \cdot \one (\eventA)
]
\leq \exp\bigg(
- \frac{M'
	\Delta_j^2 
}{ 36 \sigma^2}
\bigg).
\label{eqn-E-1}
\end{align}

\paragraph{Bounding $\cE_{2, t}$.} 
We have
\begin{align}
	\cE_{2, t} & \leq \PP [ \hat{\mean}_1 (t) \leq v_j , \pullcount_1(t) > M' ] .
\label{eqn-E-2-t-0}
\end{align}
For large $M'$, the probability should be small because if Arm 1 is pulled many times, then its empirical average reward should be close to the population mean. We will study this via concentration of self-normalized martingale \citep{APS11}.
For any $t \in \ZZ_+$, denote by $\reward_{t, 1}$ the potential reward if Arm 1 is pulled in the $t$-th round, and $\bar\history_{t-1}$ the $\sigma$-field generated by $\dataset_{t-1}$ and $\action_{t}$. Define $\eta_0 = 0$ and $\eta_t = \reward_{t, 1} - \mean_1$ for $t \geq 1$. Then, $\eta_t$ is $\bar\history_t$-measurable, $\EE (  \eta_t | \bar\history_{t-1} ) = 0$, and Assumption \ref{assumption-subg} implies that
\[
\EE ( e^{\lambda \eta_t} | \bar\history_{t-1} ) \leq e^{\lambda^2 / 2} , \qquad \forall \lambda \in \RR.
\]
Define
\begin{align*}
& V_t = M' + \sum_{i=0}^{t-1} \one ( \action_i = 1 ) = M' + N_1(t) , \\
& S_t = \sum_{i=0}^{t-1} \eta_i \one ( \action_i = 1 ) = N_1(t) [\hat{\mean}_1 (t) - \mean_1].
\end{align*}
Choose any $\delta > 0$. Theorem 1 in \cite{APS11} implies that
\begin{align*}
\PP \bigg(
|S_t|^2 / V_t \leq 2 \log ( \delta^{-1}
 \sqrt{
 V_t / M'
}
)
,~~\forall t \geq 1
\bigg) \geq 1 - \delta.
\end{align*}
When the above event happens, for all $t \geq 1$ we have
\begin{align*}
\frac{
\{ N_1(t) [\hat{\mean}_1 (t) - \mean_1] \}^2
}{
M' + N_1(t)
}
\leq 2 \log(1 / \delta) + \log \bigg(
\frac{M' + N_1(t)}{M'}
\bigg)
\leq 2 \log(1 / \delta) + \frac{N_1(t)}{M'},
\end{align*}
where the last inequality follows from the elementary fact $\log (1+z) \leq z$, $\forall z \geq 0$. Therefore, we have
\begin{align*}
\PP \bigg(
	\frac{
	 N_1(t) 
	}{
		M' + N_1(t)
	}[\hat{\mean}_1 (t) - \mean_1]^2
	\leq \frac{ 2 \log(1 / \delta) }{N_1(t)} + \frac{1}{M'}, ~~ \forall t \geq 1
\bigg)
\geq 1 - \delta.
\end{align*}
Denote by $\eventA_{\delta}$ the above event.

We now come back to \eqref{eqn-E-2-t-0}. On the event $\{ \hat{\mean}_1 (t) \leq v_j , \pullcount_1(t) > M' \}$, we have
\begin{align*}
	\frac{
	N_1(t) 
}{
	M' + N_1(t)
}[\hat{\mean}_1 (t) - \mean_1]^2 > \frac{\Delta_j^2}{18}
\qquad\text{and}\qquad
\frac{ 2 \log(1 / \delta) }{N_1(t)} + \frac{1}{M'}
< \frac{2 \log (1 / \delta) + 1}{M'}.
\end{align*}
Take $\delta = e^{1/2 - M' \Delta_j^2 / 36}$. We have $[ 2 \log (1 / \delta) + 1 ] / M' = \Delta_j^2 / 18$ and thus
\[
\PP [  \hat{\mean}_1 (t) \leq v_j  , \pullcount_1(t) > M' ]
\leq 
\PP \bigg(
	\frac{
	N_1(t) 
}{
	M' + N_1(t)
}[\hat{\mean}_1 (t) - \mean_1]^2 
> 
\frac{ 2 \log(1 / \delta) }{N_1(t)} + \frac{1}{M'}
\bigg)
\leq \PP ( \eventA_{\delta}^c ) \leq \delta.
\]
Then, \eqref{eqn-E-2-t-0} leads to
\begin{align}
\cE_{2, t} \leq e^{1/2 - M' \Delta_j^2 / 36}.
	\label{eqn-E-2}
\end{align}

\paragraph{Bounding $\cE_{3, t}$.} 
Let $\eventA$ be the event $\{ \pullcount_j(t) > M, \hat{\mean}_j(t) < u_j, 
1 \leq \pullcount_1(t) < M', \hat{\mean}_1 (t) > \hat{\mean}_j (t) \}$. Under $\eventA$, we apply the relation \eqref{eqn-ratio} and Part \ref{lem-ratio-2} of \Cref{lem-ratio} to Arms $1$ and $j$ (as the $i$ and $j$ therein) and then obtain $  \distQ_{t-1} (j)  \leq \distQ_{t-1} (1) $. Therefore, by \eqref{eqn-conditioning-1},
\begin{align}
& \cE_{3, t}
\leq  \EE [  \distQ_{t-1} (1)  \cdot \one (\eventA) ] 
\leq 	\EE \Big(
\one [ \action_t = 1, 1 \leq \pullcount_1(t) < M'  ] 
\Big) ,
\notag \\
& \sum_{t=1}^T
\cE_{3, t}
	\leq 	\EE \bigg(
\sum_{t=1}^T	\one [ \action_t = 1, 1 \leq \pullcount_1(t) < M'  ] 
	\bigg) 
\leq \lceil M' \rceil - 1.
\label{eqn-E-3}
\end{align}

\paragraph{Bounding $\cE_{4, t}$.} Let $\eventA$ be the event $\{ \pullcount_j(t) > M, \hat{\mean}_j(t) < u_j, 
1 \leq \pullcount_1(t) < M', \hat{\mean}_1 (t) \leq \hat{\mean}_j (t)  \}$. By \eqref{eqn-conditioning-1},
\[
\cE_{4, t}
=
 \EE 
\bigg(
\frac{\distQ_{t-1}(j)}{ \distQ_{t-1} (1) } 
 \distQ_{t-1} (1)  \cdot \one (\eventA) 
\bigg)
=
\EE \bigg(
\frac{\distQ_{t-1}(j)}{ \distQ_{t-1} (1) } 
\cdot \one ( \{ \action_t = 1  \} \cap \eventA) 
\bigg) .
\]
Under $\eventA$, we can apply the relation \eqref{eqn-ratio} and Part \ref{lem-ratio-3} of \Cref{lem-ratio} to Arms $j$ and $1$ (as the $i$ and $j$ therein). This yields
\begin{align*}
\frac{\distQ_{t-1}(j)}{ \distQ_{t-1} (1) } 
& \leq 
\exp\bigg(
 \frac{1}{2 \sigma^2}
\cdot
\frac{
	[ \hat\mean_j (t) - \hat\mean_1(t) ]^2
}{
	1/\pullcount_1 (t) - 1/ \pullcount_j (t)
}
\bigg)
 \leq 
\exp\bigg(
\frac{1}{2 \sigma^2}
\cdot
\frac{
[ \mu_1 - \hat\mean_1(t) ]^2
}{
	1/\pullcount_1(t) - 1/ \pullcount_j (t) 
}
\bigg) .
\end{align*}
Since $1 \leq \pullcount_1(t) < (1-c) M < M < \pullcount_j(t)$, we have $\pullcount_j(t) > \pullcount_1(t) / (1-c)$ and
\[
\frac{1}{\pullcount_1(t)} - \frac{1}{\pullcount_j(t)} \geq 
\frac{1}{\pullcount_1(t)} - \frac{1}{\pullcount_1(t) / (1-c)} =  \frac{c}{\pullcount_1(t)}.
\]
Hence,
\begin{align*}
&
\frac{ \distQ_{t-1} (j) }
{ 
	\distQ_{t-1} (1)
}
	\leq 
	\exp\bigg(
	\frac{ \pullcount_1(t)
		[ \mu_1 - \hat\mean_1(t) ]^2
	}{ 2 c \sigma^2}
	\bigg) .
\end{align*}
We have
\begin{align*}
\sum_{t=1}^{T}
\cE_{4, t} 
& \leq
\EE 
\bigg[
\sum_{t=1}^{T}
\exp\bigg(
\frac{ \pullcount_1(t)
	[  \hat\mean_1(t) - \mu_1 ]^2
}{2 c \sigma^2}
\bigg)
\one [ \action_t = 1, 1 \leq \pullcount_1(t) < M'  ] 
\bigg] \\
& \leq \EE 
\bigg[
\sum_{k=2}^{ \lceil M' \rceil - 1}
\exp\bigg(
\frac{ (k-1)
	[  \hat\mean_1( \tau_{1,k} ) - \mu_1 ]^2
}{2 c \sigma^2}
\bigg)
\bigg]
 = \sum_{s = 1}^{ \lceil M' \rceil - 2} \EE 
\bigg[
\exp\bigg(
\frac{ s
( \xi_{1,s}  - \mu_1 )^2
}{2 c \sigma^2}
\bigg)
\bigg] ,
\end{align*}
where $\tau_{1,k}$ is the time of the $k$-th pull of Arm $1$, see \Cref{defn-filtration}. 

When $c \sigma^2 > 1$, we obtain from \Cref{lem-tail} that
\[
 \EE 
\bigg[
\exp\bigg(
\frac{ s
	( \xi_{1,s}  - \mu_1 )^2
}{2 c \sigma^2}
\bigg)
\bigg]
\leq \frac{1}{ \sqrt{ 1 - 1 / (c \sigma^2) } } , \qquad \forall s \in \ZZ_+.
\]
Therefore,
\begin{align}
	\sum_{t=1}^{T}
	\cE_{4, t} 
	& \leq
\frac{  \lceil M' \rceil - 2 }{ \sqrt{ 1 - 1 / (c \sigma^2) } } .
	\label{eqn-E-4}
\end{align}

\paragraph{Bounding $\cE_{5, t}$.} If $\pullcount_1(t) = 0 $, then $\Lambda ( 1 , \dataset_{t-1}) = 0$. The relation \eqref{eqn-ratio} yields
\begin{align*}
\frac{ \distQ_{t-1} (j) }
{ 
\distQ_{t-1} (1)
}
 = e^{  \Lambda ( 1 , \dataset_{t-1}) - \Lambda (j, \dataset_{t-1}) } \leq 1.
\end{align*}
Then, by \eqref{eqn-conditioning-1},
\begin{align*}
\cE_{5, t} 
& \leq \PP [ \action_t = j,  \pullcount_1(t) = 0 ] 
= \EE [ \distQ_{t-1} (j) \cdot \one ( \pullcount_1(t) = 0 ) ] 
 \\&
 \leq \EE [ \distQ_{t-1} (1) \cdot \one ( \pullcount_1(t) = 0 ) ] 
 = \EE \Big(
\one [ \action_t = 1 ,  \pullcount_1(t) = 0 ] 
\Big).
\end{align*}
We have
\begin{align}
\sum_{t=1}^{T} \cE_{5, t}
\leq 
 \EE \bigg(
\sum_{t=1}^{T}
\one [ \action_t = 1 ,  \pullcount_1(t) = 0 ] 
\bigg)
\leq 1.
\label{eqn-E-5}
\end{align}

\paragraph{Bounding $\cJ_{3}$.} 
Below we use $\lesssim$ to hide universal constant factors. Summarizing \eqref{eqn-E-1}, \eqref{eqn-E-2}, \eqref{eqn-E-3}, \eqref{eqn-E-4} and \eqref{eqn-E-5}, we get
\begin{align*}
\cJ_3
& \leq 
T  \exp\bigg(
- \frac{M'
	\Delta_j^2 
}{ 36 \sigma^2}
\bigg)
+ 
T e^{1/2 - M' \Delta_j^2 / 36}
+ (
\lceil M' \rceil - 1
)
+ 
\frac{  \lceil M' \rceil - 2 }{ \sqrt{ 1 - 1 / (c \sigma^2) } } 
+ 1 \notag \\
&\lesssim
T   \exp\bigg(
- \frac{M'
	\Delta_j^2 
}{ 36 \sigma^2}
\bigg)
+ \frac{ (M' + 1) }{
\sqrt{ 1 - 1 / (c \sigma^2) } 
} 
=
T \exp\bigg(
- \frac{ (1-c) M
	\Delta_j^2 
}{ 36 \sigma^2}
\bigg)
+ \frac{  [ (1-c) M + 1] }{
	\sqrt{ 1 - 1 / (c \sigma^2) } 
} ,
\end{align*}
so long as $1 / \sigma^2 < c < 1$. Let $c = (1 + \sigma^{-2}) / 2$. We have $1 - c = (1 - \sigma^{-2})/2$ and $1 - 1/ c \sigma^2 = (\sigma^2 - 1) / (\sigma^2 + 1) \geq (1 - \sigma^{-2})/2$. Hence,
\begin{align}
	\cJ_3
	&  \lesssim
T
\exp\bigg(
- \frac{ (1-\sigma^{-2}) M
	\Delta_j^2 
}{ 72 \sigma^2}
\bigg)
+ \frac{
1
}{
\sqrt{ 1 - \sigma^{-2} }
}
\bigg(
\frac{1 - \sigma^{-2}}{2} M + 1
\bigg) \notag\\
&  \lesssim
T
  \exp\bigg(
- \frac{ (1-\sigma^{-2}) M
	\Delta_j^2 
}{ 72 \sigma^2}
\bigg)
+  M +
\frac{
1
}{
	\sqrt{ 1 - \sigma^{-2} }
}
.
	\label{eqn-J-1}
\end{align}

\subsubsection{Final steps}

Combining \eqref{eqn-J-0}, \eqref{eqn-J-3}, \eqref{eqn-J-2} and \eqref{eqn-J-1}, we get
\begin{align*}
 \sum_{t=1}^{T}  \PP ( \action_t = j ) 
& \lesssim
\bigg[
T  \exp\bigg(
- \frac{ (1-\sigma^{-2}) M
	\Delta_j^2 
}{ 72 \sigma^2}
\bigg)
+ M +
\frac{
1
}{
	\sqrt{ 1 - \sigma^{-2} }
}
\bigg]
+ \bigg(
 \frac{1}{\Delta_j^2} + 1
\bigg)
+ (M+1) \\
& \lesssim
T 
 \exp\bigg(
- \frac{ (1-\sigma^{-2}) M
	\Delta_j^2 
}{ 72 \sigma^2}
\bigg)
+  \frac{1}{\Delta_j^2} + M  + \frac{1}{ \sqrt{1 - \sigma^{-2}} } 
, \qquad \forall M > 0.
\end{align*}
By taking
\[
M = \frac{72 \sigma^2}{1 - \sigma^{-2}} \cdot  \frac{ \log ( \max\{ T \Delta_j^2 , e \} }{ \Delta_j^2 } ,
\]
we get
\begin{align*}
\sum_{t=1}^{T}  \PP ( \action_t = j ) 
& \lesssim
T e^{ - \log (  \max\{ T \Delta_j^2 , e \} ) }
+ \frac{1}{\Delta_j^2}
	+ 
 \frac{ \sigma^2}{1 - \sigma^{-2}} \cdot  \frac{ \log (  \max\{ T \Delta_j^2 , e \} ) }{ \Delta_j^2 } 
  + \frac{1}{ \sqrt{1 - \sigma^{-2}} }
 \\
& =   \frac{ \sigma^2}{1 - \sigma^{-2}} \cdot  \frac{ \log ( \max\{ T \Delta_j^2 , e \} ) }{ \Delta_j^2 } 
+ \frac{1}{ \sqrt{1 - \sigma^{-2}} }
.
\end{align*}

\subsection{Proof of \Cref{lem-ratio}}\label{sec-lem-ratio-proof}

\subsubsection{Part \ref{lem-ratio-1}}

For notational simplicity, we will suppress the time index $t$ in $\pullcount_k(t)$'s and $\hat{\mean}_k(t)$'s. 
The result is trivial when $\hat\mean_i = \hat\mean_j$. Below we assume that $\hat\mean_i > \hat\mean_j$. 
By \eqref{eqn-Lambda}, we have
\begin{align}
	&	2 \sigma^2 \Lambda (j, \dataset_{t-1})
	= \min_{\btheta \in \spaceofparameter_j } 
	\bigg\{
	\sum_{k=1}^{K} 
	\pullcount_k   (  \hat\mean_k  - \theta_k )^2 
	\bigg\} \notag\\
	& \geq  \min_{\btheta \in \spaceofparameter_j } 
	\bigg\{
	\pullcount_j   (  \hat\mean_j   - \theta_j )^2 
	+ 
	\pullcount_{i}   ( \hat\mean_{i}  - \theta_{i} )^2 
	\bigg\} 
	=   \min_{ \theta_j \geq \theta_{i} } 
	\bigg\{
	\pullcount_j   ( \hat\mean_j   - \theta_j )^2 
	+ 
	\pullcount_{i}   ( \hat\mean_{i}  - \theta_{i} )^2 
	\bigg\}.
\label{eqn-proof-ratio-1}
\end{align}
Denote by $h(\theta_j, \theta_i)$ the function in the brackets. The assumption $\hat{\mean}_i  > \hat{\mean}_j $ implies that for any $\theta_j$,
\[
\min_{ \theta_i \leq \theta_j }  
h(\theta_j, \theta_i)
= h \Big( \theta_j, \min \{ \hat{\mean}_i  , \theta_j \} \Big)
= \pullcount_j   (  \hat\mean_j   - \theta_j )^2 
+ 
\pullcount_{i}   ( \hat\mean_{i}  - \theta_{j} )_+^2 
.
\]
View the above as a function of $ \theta_{j} $. It is strictly increasing on $( \hat\mean_{i}  , +\infty)$. On the complement set $(-\infty, \hat\mean_{i}  ]$, the expression simplies to $\pullcount_j   ( \hat\mean_j   - \theta_j )^2 
+ 
\pullcount_{i}   (\hat\mean_{i}  - \theta_{j} )^2 $. This function's minimizer and minimum value are
\[
\frac{ \pullcount_j    \hat\mean_j    + \pullcount_{i}    \hat\mean_{i}   }{  \pullcount_j    + \pullcount_{i}  }
\qquad\text{and}\qquad
\frac{ (  \hat\mean_{i}   - \hat\mean_j   )^2 }{  1/ \pullcount_{i}   + 1/ \pullcount_j   } .
\] 
This fact and \eqref{eqn-proof-ratio-1} lead to the desired inequality. 

\subsubsection{Part \ref{lem-ratio-2}}

Choose any $ \bar\btheta \in 
\argmin_{\btheta \in \spaceofparameter_j } \ell (\btheta, \dataset_{t-1})$.

\paragraph{Case 1: $\bar\theta_j \geq \hat\mean_i  $.}
It is easily seen that $\bar\theta_i =  \hat\mean_i$. Define $\bm{\eta} \in \RR^K$ by
\[
\eta_k = \begin{cases}
	\hat\mean_j & , \mbox{ if } k = j  \\
	\bar\theta_j & , \mbox{ if } k = i \\
	\bar\theta_k & , \mbox{ otherwise }
\end{cases}.
\]
We have $\bm{\eta} \in \spaceofparameter_i$ and
\begin{align*}
	& \Lambda ( i, \dataset_{t-1}) - \Lambda (j, \dataset_{t-1}) 
	= \min_{\btheta \in \spaceofparameter_i } 
	\ell (\btheta, \dataset_{t-1})
	-
	\min_{\btheta \in \spaceofparameter_j } 
	\ell (\btheta, \dataset_{t-1})
	\leq 
	\ell ( \bm\eta, \dataset_{t-1})
	- \ell (\bar\btheta, \dataset_{t-1})
	\\
	& = 
	\frac{1}{2\sigma^2}
	\bigg(
	\pullcount_i   (\hat\mean_i  - \eta_i )^2 
	+
	\pullcount_j   (\hat\mean_j  - \eta_j )^2 
	\bigg)
	-
	\frac{1}{2\sigma^2} 
	\bigg(
	\pullcount_i   (\hat\mean_i  - \bar\theta_i )^2 
	+
	\pullcount_j   ( \hat\mean_j  - \bar\theta_j )^2 
	\bigg)
	\\
	& = 
	\frac{1}{2\sigma^2}
	\bigg(
	\pullcount_i   (\hat\mean_i  - \bar\theta_j )^2  - \pullcount_j   ( \hat\mean_j  - \bar\theta_j )^2 
	\bigg) \leq 0.
\end{align*}
The last inequality follows from $\bar\theta_j \geq \hat\mean_i \geq \hat\mean_j $ and $\pullcount_j \geq \pullcount_i$.

\paragraph{Case 2: $\bar\theta_j < \hat\mean_i  $.}
It is easily seen that $\hat\mean_j \leq \bar\theta_j = \bar\theta_i$. Define $\bm{\eta} \in \RR^K$ by
\[
\eta_k = \begin{cases}
	\hat\mean_j  & , \mbox{ if } k = j  \\
	\hat\mean_i & , \mbox{ if } k = i \\
	\bar\theta_k & , \mbox{ otherwise }
\end{cases}.
\]
We have $\bm{\eta} \in \spaceofparameter_i$ and
\begin{align*}
	& \Lambda ( i, \dataset_{t-1}) - \Lambda (j, \dataset_{t-1}) 
	= \min_{\btheta \in \spaceofparameter_i } 
	\ell (\btheta, \dataset_{t-1})
	-
	\min_{\btheta \in \spaceofparameter_j } 
	\ell (\btheta, \dataset_{t-1})
	\leq 
	\ell ( \bm\eta, \dataset_{t-1})
	- \ell (\bar\btheta, \dataset_{t-1})
	\\
	& = 
	\frac{1}{2\sigma^2}
	\bigg(
	\pullcount_i   (\hat\mean_i  - \eta_i )^2 
	+
	\pullcount_j   (\hat\mean_j  - \eta_j )^2 
	\bigg)
	-
	\frac{1}{2\sigma^2} 
	\bigg(
	\pullcount_i   (\hat\mean_i  - \bar\theta_i )^2 
	+
	\pullcount_j   ( \hat\mean_j  - \bar\theta_j )^2 
	\bigg)
	\\
	& = 
	0 - \frac{1}{2\sigma^2}
	\bigg(
	\pullcount_i   (\hat\mean_i  - \bar\theta_j )^2  + \pullcount_j   ( \hat\mean_j  - \bar\theta_j )^2 
	\bigg) \leq 0.
\end{align*}

\subsubsection{Part \ref{lem-ratio-3}}

Choose any $\bar\btheta \in 
\argmin_{\btheta \in \spaceofparameter_i } \ell (\btheta, \dataset_{t-1})$. We invoke a useful result.

\begin{claim}\label{claim-lem-ratio-3}
$\bar\theta_i \geq \hat{\mean}_i   \geq \hat{\mean}_j   = \bar\theta_j $
\end{claim}

\begin{proof}[\bf Proof of Claim \ref{claim-lem-ratio-3}]
Define $\bm{\eta} \in \RR^K$ by
\[
\eta_k = \begin{cases}
\max\{ \bar\theta_i ,	\hat{\mean}_i  \}& , \mbox{ if } k = i \\
\hat\mean_j & , \mbox{ if } k = j  \\
	\bar\theta_k & , \mbox{ otherwise }
\end{cases}.
\]
We have $\bm{\eta} \in \spaceofparameter_i$ and
\begin{align*}
0 \geq 2 \sigma^2 [ \ell (\bar\btheta, \dataset_{t-1}) - \ell (\bm{\eta} , \dataset_{t-1}) ]
& = 
[\pullcount_i   (\hat\mean_i  - \bar\theta_i )^2 
+
\pullcount_j   ( \hat\mean_j  - \bar\theta_j )^2 ]
-
[\pullcount_i   (\hat\mean_i  - \eta_i )^2 
+
\pullcount_j   (\hat\mean_j  - \eta_j )^2 ]
 \\
& = \pullcount_i   (\hat\mean_i  - \bar\theta_i )_-^2 
+ \pullcount_j   ( \hat\mean_j  - \bar\theta_j )^2 
\end{align*}
The inequality forces $\bar\theta_i \geq \hat\mean_i $ and $ \bar\theta_j  = \hat\mean_j $.
\end{proof}

We now come back to the main proof. Define $\bm{\eta} \in \RR^K$ by
\[
\eta_k = \begin{cases}
	\bar\theta_i & , \mbox{ if } k = j  \\
	\hat{\mean}_i  & , \mbox{ if } k = i \\
	\bar\theta_k & , \mbox{ otherwise }
\end{cases}.
\]
We have $\bm{\eta} \in \spaceofparameter_j$ and
\begin{align*}
	& \Lambda ( i, \dataset_{t-1}) - \Lambda (j, \dataset_{t-1}) 
	= \min_{\btheta \in \spaceofparameter_i } 
	\ell (\btheta, \dataset_{t-1})
	-
	\min_{\btheta \in \spaceofparameter_j } 
	\ell (\btheta, \dataset_{t-1})
	\geq 
	\ell ( \bar\btheta, \dataset_{t-1})
	- \ell (\bm{\eta}, \dataset_{t-1})
	\\
	& = 
	\frac{1}{2\sigma^2}
	\bigg(
	\pullcount_i   (\hat\mean_i  - \bar\theta_i )^2 
	+
	\pullcount_j   ( \hat\mean_j  - \bar\theta_j )^2 
	\bigg)
	-
	\frac{1}{2\sigma^2} \bigg(
	\pullcount_i   (\hat\mean_i  - \eta_i )^2 
	+
	\pullcount_j   (\hat\mean_j  - \eta_j )^2 
	\bigg)
	\\
	& = 
	\frac{1}{2\sigma^2}
	\bigg(
	\pullcount_i   (\hat\mean_i  - \bar\theta_i )^2  - \pullcount_j   ( \hat\mean_j  - \bar\theta_i )^2 
	\bigg) 
	\geq \frac{1}{2 \sigma^2} \inf_{z \geq \hat\mean_i } \bigg\{
	\pullcount_i   (\hat\mean_i  - z )^2  - \pullcount_j   (\hat\mean_j  - z )^2 
	\bigg\} \\
	& =  \frac{1}{2\sigma^2} \inf_{z \geq 0} 
	\bigg\{
	\pullcount_i    z^2  - \pullcount_j   [ z + (\hat\mean_i  - \hat\mean_j ) ]^2 
	\bigg\}.
\end{align*}

Denote by $g(z)$ the function in the bracket. From
\[
g'(z) / 2 = \pullcount_i  z - \pullcount_j  [ z + ( \hat\mean_i  - \hat\mean_j  ) ]
= ( \pullcount_i  - \pullcount_j  ) z - \pullcount_j  ( \hat\mean_i  - \hat\mean_j  ) .
\]
and $\pullcount_i  > \pullcount_j  $, we derive that
\[
\inf_{z \geq 0} g(z) 
= g \bigg(
\frac{\pullcount_j (  \hat\mean_i  - \hat\mean_j  ) }{\pullcount_i - \pullcount_j}
\bigg)
= - \frac{\pullcount_i \pullcount_j }{\pullcount_i - \pullcount_j} (  \hat\mean_i  - \hat\mean_j  )^2
= - \frac{
	( \hat\mean_i  - \hat\mean_j  )^2
}{
	1/\pullcount_j - 1/ \pullcount_i
}.
\]

{
\bibliographystyle{ims}
\bibliography{bib}
}

\end{document}